\documentclass[10pt,journal,compsoc]{IEEEtran}
\ifCLASSOPTIONcompsoc
    \usepackage[nocompress]{cite}
\else
    \usepackage{cite}
\fi
\ifCLASSINFOpdf
    \usepackage[pdftex]{graphicx}
\else
    \usepackage[dvips]{graphicx}
\fi
    \usepackage[cmex10]{amsmath}
    \usepackage{amsfonts}
    \usepackage{amssymb}
    \usepackage{amsthm}
    \usepackage{bbm}
    \usepackage{blkarray}
    \newtheorem{theorem}{Theorem}
    \newtheorem{lemma}{Lemma}
    \newtheorem{corollary}{Corollary}
    
    \newtheorem{definition}{Definition}

    \newtheorem{example}{Example}

    \usepackage{enumitem}

    \usepackage{algorithm}
    \usepackage{algorithmic}

    \usepackage{array}

\usepackage{stfloats}

\ifCLASSOPTIONcaptionsoff
    \usepackage[nomarkers]{endfloat}
    \let\MYoriglatexcaption\caption
    \renewcommand{\caption}[2][\relax]{\MYoriglatexcaption[#2]{#2}}
\fi
    \usepackage{url}
    \usepackage{booktabs}
    \usepackage[utf8]{inputenc}
    \usepackage{multirow}

  \usepackage{xcolor}

  \usepackage{subcaption}

  \usepackage{ragged2e}

  \renewcommand{\raggedright}{\leftskip=0pt \rightskip=0pt plus 0cm}

\begin{document}
%
\title{Fast Stochastic Ordinal Embedding with Variance Reduction and Adaptive Step Size}
%
%
%
%

\author
{
  Ke~Ma,~\IEEEmembership{Member,~IEEE,}
  Jinshan~Zeng,~
  Jiechao~Xiong,~
  Qianqian~Xu,~\IEEEmembership{Senior Member,~IEEE,}
  Xiaochun~Cao$^*$,~\IEEEmembership{Senior Member,~IEEE,}
  Wei~Liu,~\IEEEmembership{Senior Member,~IEEE,}
  Yuan~Yao$^*$
  \IEEEcompsocitemizethanks
  {
    \IEEEcompsocthanksitem K. Ma is with the School of Computer Science and Technology, University of Chinese Academy of Sciences, Beijing 100049, China, and with the Artificial Intelligence Research Center, Peng Cheng Laboratory, Shenzhen 518055, China,  and part of this work was performed when he was in the Key Laboratory of Information Security, Institute of Information Engineering, Chinese Academy of Sciences, Beijing 100093, China, and in the School of Cyber Security, University of Chinese Academy of Sciences, Beijing 100049, China. E-mail: make@ucas.ac.cn\protect\\
    \IEEEcompsocthanksitem J. Zeng is with the School of Computer Information Engineering, Jiangxi Normal University, Nanchang, Jiangxi 330022, China, and part of this work was performed when he was with the Department of Mathematics, Hong Kong University of Science and Technology, Clear Water Bay, Kowloon, Hong Kong. E-mail: jsh.zeng@gmail.com\protect\\
    \IEEEcompsocthanksitem J. Xiong is with the Tencent AI Lab, Shenzhen, Guangdong, China. E-mail: jcxiong@tencent.com\protect\\
    \IEEEcompsocthanksitem Q. Xu is with the Key Laboratory of Intelligent Information Processing, Institute of Computing Technology, Chinese Academy of Sciences, Beijing 100190, China. E-mail: qianqian.xu@vipl.ict.ac.cn, xuqianqian@ict.ac.cn\protect\\
    \IEEEcompsocthanksitem X. Cao is with the State Key Laboratory of Information Security, Institute of Information Engineering, Chinese Academy of Sciences, Beijing, 100093, China, and with the Cyberspace Security Research Center, Peng Cheng Laboratory, Shenzhen 518055, China, and with the School of Cyber Security, University of Chinese Academy of Sciences, Beijing 100049, China. Corresponding Author. E-mail: caoxiaochun@iie.ac.cn\protect\\
    \IEEEcompsocthanksitem W. Liu is with the Tencent AI Lab, Shenzhen, Guangdong, China. E-mail: wl2223@columbia.edu\protect\\
    \IEEEcompsocthanksitem Y. Yao is with the Department of Mathematics, and by courtesy, the Department of Computer Science and Engineering, Hong Kong University of Science and Technology, Clear Water Bay, Kowloon, Hong Kong. Corresponding Author. E-mail: yuany@ust.hk\protect\\
  }
  \thanks{Manuscript received September 05, 2018; revised June 12, 2019.}
}

\markboth{IEEE TRANSACTIONS ON KNOWLEDGE AND DATA ENGINEERING,~Vol.~0, No.~0, August~2019}%
{Shell \MakeLowercase{\textit{et al.}}: Bare Demo of IEEEtran.cls for Computer Society Journals}
%

\IEEEtitleabstractindextext
{
	\begin{abstract}
  \justifying
  Learning representation from relative similarity comparisons, often called ordinal embedding, gains rising attention in recent years. Most of the existing methods are based on semi-definite programming (\textit{SDP}), which is generally time-consuming and degrades the scalability, especially confronting large-scale data. To overcome this challenge, we propose a stochastic algorithm called \textit{SVRG-SBB}, which has the following features: i) achieving good scalability via dropping positive semi-definite (\textit{PSD}) constraints as serving a fast algorithm, i.e., stochastic variance reduced gradient (\textit{SVRG}) method, and ii) adaptive learning via introducing a new, adaptive step size called the stabilized Barzilai-Borwein (\textit{SBB}) step size. Theoretically, under some natural assumptions, we show the $\boldsymbol{O}(\frac{1}{T})$ rate of convergence to a stationary point of the proposed algorithm, where $T$ is the number of total iterations. Under the further Polyak-\L{}ojasiewicz assumption, we can show the global linear convergence (i.e., exponentially fast converging to a global optimum) of the proposed algorithm. Numerous simulations and real-world data experiments are conducted to show the effectiveness of the proposed algorithm by comparing with the state-of-the-art methods, notably, much lower computational cost with good prediction performance.
	\end{abstract}

}

\maketitle

\IEEEdisplaynontitleabstractindextext

%
\IEEEpeerreviewmaketitle

\IEEEraisesectionheading{\section{Introduction}
\label{sec:introduction}}

\IEEEPARstart{O}{rdinal} embedding aims to learn the representation of data as points in a low-dimensional embedded space. Here the ``low-dimensional'' means the embedding dimension is much smaller than the number of data points. The distances between these points agree with a set of relative similarity comparisons. Relative comparisons are often collected via the participators who are asked to answer questions like:

\emph{``Is the similarity between object $i$ and $j$ larger than the similarity between $l$ and $k$?"}

The feedback of these questions provide us with a set of quadruplets, \textit{i.e.}, $(i,j,l,k)$ which indicates that the similarity between object $i$ and $j$ is larger than the similarity between $l$ and $k$. These relative similarity comparisons are the supervision information for ordinal embedding. Without prior knowledge, the relative similarity comparisons always involve all objects, and the number of potential quadruplets could be $\boldsymbol{O}(n^4)$. Even under the so-called ``local'' setting where we restrict $l=i$, the triple-wise comparisons, $(i,j,k)$, also have the complexity $\boldsymbol{O}(n^3)$.

The ordinal embedding problem was firstly studied by \cite{Shepard1962a,Shepard1962b,Kruskal1964a,Kruskal1964b} in the psychometric society. In recent years, it has drawn a lot of attention in machine learning \cite{jamieson2011low,Ailon:2012:ALA:2503308.2188390,53e99af7b7602d97023851bf,2015arXiv150102861A,amid2015multiview,NIPS2016_6554}, statistic ranking \cite{McFee:2011:LMS:1953048.1953063,kevin2011active,NIPS2012_0599}, artificial intelligence \cite{Heikinheimo2013TheCA,503}, information retrieval \cite{7410580}, and computer vision \cite{wah2014similarity,wilberKKB2015concept}, etc.

Most of the ordinal embedding methods are based on the semi-definite programming (\textit{SDP}). Some typical methods include the Generalized Non-Metric Multidimensional Scaling (\textit{GNMDS}) \cite{agarwal2007generalized}, Crowd Kernel Learning (\textit{CKL}) \cite{tamuz2011adaptiive}, and Stochastic Triplet Embedding (\textit{STE/TSTE}) \cite{vandermaaten2012stochastic}. The main idea of such methods is to formulate the ordinal embedding problem into a convex, low-rank \textit{SDP} problem with respect to the Gram matrix of the embedding points. In order to solve such a \textit{SDP} problem, the traditional methods generally employ the projection gradient descent to satisfy the positive semi-definite constraint, where the singular value decomposition (\textit{SVD}) is required at each iteration. This inhibits the popularity of this type of methods for large-scale and online ordinal embedding applications.

To handle the large-scale ordinal embedding problem, we reformulate the considered problem using the embedding matrix instead of its Gram matrix. By taking advantage of this new non-convex formulation, the positive semi-definite constraint is eliminated. Furthermore, we exploit the well-known stochastic variance reduced gradient (\textit{SVRG}) method to efficiently solve the developed formulation, which is a fast stochastic algorithm proposed in \cite{rie2013accelerating}. Generally, step size, one essential hyper-parameter, should be tuned in \textit{SVRG}. It is a difficult task in practice as the Lipschitz constant is hard to estimate. To facilitate the use of \textit{SVRG}, Tan et al. \cite{NIPS2016_6286} introduced the well-known, adaptive step size called the Barzilai-Borwein (\textit{BB}) step size \cite{barzilai1988two}, and proved its linear convergence in the strongly convex case. However, as shown in our simulations (see, Figure \ref{fig:step}), the absolute value of the original \textit{BB} step size varies dramatically regarding the epoch number, when applied to our developed ordinal embedding formulation. One major reason is that our developed ordinal embedding model is not strongly convex, and even non-convex. Thus, in such setting, the denominator of \textit{BB} step size might be very close to zero, leading to the instability of the \textit{BB} step size. We add another positive term to the non-negative denominator of \textit{BB} step size which overcomes such instability of the original \textit{BB} step size. Similar to the original version, the new step size is adaptive with almost the same computational cost. More importantly, the new step size is more stable than the original \textit{BB} step size, and can be applied to more general case beyond the strongly convexity assumption. Henceforth, we call the new method as \textit{stabilized Barzilai-Borwein (SBB)} step size. By incorporating the \textit{SBB} step size with \textit{SVRG}, we propose a new stochastic algorithm called \textit{SVRG-SBB} for efficiently solving the considered ordinal embedding model.

In summary, our main contributions can be shown as follows:
\begin{itemize}[leftmargin=*]
\item
{We propose a non-convex framework for the ordinal embedding problem via considering the original embedding variable rather than its Gram matrix. We get rid of the positive semi-definite (\textit{PSD}) constraint on the Gram matrix, and thus, our proposed algorithm is \textit{SVD}-free and has better scalability than the existing convex ordinal embedding methods.}

\item
{The introduced \textit{SBB} step size can overcome the instability of the original \textit{BB} which comes from the absence of strongly convexity. More importantly, the proposed \textit{SVRG-SBB} algorithm outperforms most of the state-of-the-art methods as shown by numerous simulations and real-world data experiments, in the sense that \textit{SVRG-SBB} often significantly reduces the computational cost.}

\item
{We establish $\boldsymbol{O}(\frac{1}{T})$ convergence rate of \textit{SVRG-SBB} in the sense of converging to a stationary point, where $T$ is the total number of iterations. Such result is comparable with the existing convergence results in the literature.}
\end{itemize}

\begin{figure}[thb!]
	\centering
	\includegraphics[width = 0.6\columnwidth]{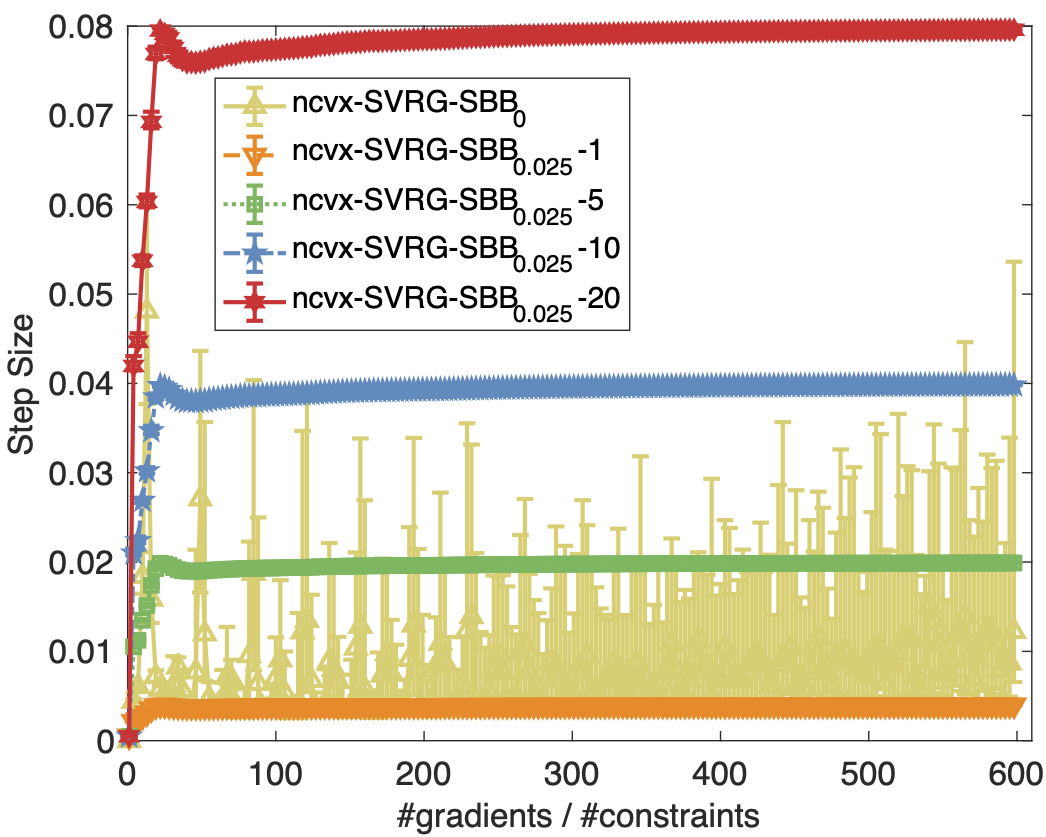}
	\caption{Step sizes along iterations of SVRG-SBB$_\epsilon$ on the synthetic data, where the dark yellow curve of \textbf{ncvx-SVRG-SBB$_0$} is exactly the varying of the \textit{BB} step size in this setting.}
	\label{fig:step}
\end{figure}

This paper is an extension of our conference work \cite{DBLP:conf/aaai/MaZXXCLY18}, where we propose the basic \textit{SVRG-SBB} method which derives the adaptive step size. But there still exist some limitations in our conference method. First, the original \textit{SVRG-SBB} does not incorporate with mini-batch paradigm which provides a computationally efficient process than single point update. Second, we observe that the well-known ``local optimal'' of non-convex problem does not have serious impact on the embedding result. The empirical success in non-convex ordinal embedding poses a new problem that under what conditions the non-convex stochastic algorithms may find the global optima effectively. We provide a possible answer of this question with the help of the Polyak-\L{}ojasiewicz (\textit{PL}) condition. Finally, we summarize the existing ordinal embedding method and propose the generalized ordinal embedding framework which generalizes the existing classification-based methods including \textit{GNMDS, CKL} and \textit{STE/TSTE}. We hope the new framework will guide the future research directions.
\\\\
\textbf{Organization}
\\\\
The remainder of the paper is organized as follows.
In Section 2, we describe the mathematical formulation of the generalized ordinal embedding problem.
Section 3 shows the development of the \textit{SVRG-SBB} algorithm for non-convex ordinal embedding.
Section 4 establishes the convergence analysis of the proposed algorithm.
Comprehensive experimental validation based on simulated and real-world datasets are demonstrated in Section 5. We conclude this paper in Section 6.

\section{Generalized Ordinal Embedding}

Throughout the paper, we denote scalars, vectors, matrices and sets as lower case letters ($x$), bold lower case letters ($\boldsymbol{x}$), bold capital letters ($\boldsymbol{X}$) and calligraphy upper case letters ($\mathcal{X}$). $x_i$ and $x_{ij}$ denote the $i_{\text{th}}$ element of vector $\boldsymbol{x}$ and $(i, j)$ entry of matrix $\boldsymbol{X}$, respectively. For any $\boldsymbol{x}\in\mathbb{R}^{p}$, $\|\boldsymbol{x}\|_2$ denotes its $\ell_2$ norm. $\boldsymbol{I}_n$ is the identity matrix with size $n\times n$ and the subscript $n$ would be omitted if there is no confusion. For any $\boldsymbol{X}\in\mathbb{R}^{p\times n}$, $\|\boldsymbol{X}\|_F$ and $\textit{rank}(\boldsymbol{X})$ denote the Frobenius norm and rank of $\boldsymbol{X}$. $\textit{vec}(\boldsymbol{X})$ is the vectorization operator on $\boldsymbol{X}$ by column. For any square matrix $\boldsymbol{G}\in\mathbb{R}^{n\times n}$, $\textit{tr}(\boldsymbol{G})$ is the trace of $\boldsymbol{G}$. $[n]$ is the set of $\{1,\dots,n\}$. For any $\boldsymbol{X}\in\mathbb{R}^{p\times n}$, $\boldsymbol{G}=\boldsymbol{X}^{\top}\boldsymbol{X}$ is the Gram matrix. For any $(\boldsymbol{x}_i,\boldsymbol{x}_j)\in\mathcal{X}\times\mathcal{X}$ where $\mathcal{X}\subset\mathbb{R}^p$, $d_{ij}=d(\boldsymbol{x}_i,\boldsymbol{x}_j)$ is the distance between $\boldsymbol{x}_i$ and $\boldsymbol{x}_j$, and $\boldsymbol{D}=\{d^2(\boldsymbol{x}_i,\boldsymbol{x}_i)\}$ is the squared distance matrix of $\boldsymbol{X}$. Here the distance $d:\mathbb{R}^{p}\times\mathbb{R}^{p}\rightarrow\mathbb{R}_{+}$ depends on the embedded space. In case of the Euclidean space, we adopt the Euclidean distance if not specified. $\mathbb{E}[\cdot]$ represents the expectation.

Let $\mathcal{O} = \{\boldsymbol{o}_1,\dots,\boldsymbol{o}_n\}$ be a collection of objects, $\mathcal{X} \subset \mathbb{R}^p$ be a low-dimensional embedded space where $p\ll n$, and $\psi^*:\mathcal{O}\times\mathcal{O}\rightarrow\mathbb{R}_+$ be a dissimilarity function of $\mathcal{O}$ where $\psi^*_{ij}$ is the dissimilarity measure between $o_i$ and $o_j$. The traditional multi-dimensional scaling (\textit{MDS}) methods embed $\mathcal{O}$ into $\mathcal{X}$ based on $\Psi^*=\{\psi^*_{ij},\ i,j\in[n], i\neq j\}$. However, there is always a lack of a dissimilarity function $\psi^*$ that can evaluate the objects $\mathcal{O}$ properly for real-world applications, e.g., \cite{503,wilberKKB2015concept,DBLP:conf/icdm/Ukkonen17,DBLP:conf/nips/MasonJN17}. As an alternative, ordinal embedding methods incorporate human knowledge into the loop and relax the requirement of $\Psi^*$.

By collecting a partially ordered set which assesses dissimilarity on a relative scale by human, ordinal embedding methods establish relative dissimilarity of $\mathcal{X}$ and obtain embedding $\boldsymbol{X} =\{\boldsymbol{x}_i:\ \boldsymbol{x}_i\in\mathcal{X},\ \boldsymbol{o}_i\in\mathcal{O},\ i\in[n]\}$

based on the dissimilarity comparisons. Specifically, given a dissimilarity function $\phi:\mathcal{X}\times\mathcal{X}\rightarrow\mathbb{R}_+$ and $\phi_{ij}=\phi(\boldsymbol{x}_i,\boldsymbol{x}_j)$ is the dissimilarity between $\boldsymbol{x}_i$ and $\boldsymbol{x}_j$, we collect a set of quadruplets, that is,
\begin{equation}
	\begin{aligned}
		& \mathcal{Q} &=&\ \ \left\{ q\ |\ q=(i,j,l,k),\ (\phi_{ij},\ \phi_{lk})\in\Phi^{2}\right\},
	\end{aligned}
\end{equation}
and define $\Phi^2$ as
\begin{equation}
	\begin{aligned}
		& \Phi^2 &=&\ \ \left\{(\phi_{ij},  \phi_{lk}) \ |\ \phi_{ij}<\phi_{lk},\ i,\ j,\ l,\ k\in[n], \right.\\
		& & &\ \ \ \ \ \ \ \ \ \ \ \ \ \ \ \ \ \ \ \ \left.i\neq j,\ l\neq k,\ (i,j)\neq(l,k)\right\}.
	\end{aligned}
\end{equation}
Although the embedding $\boldsymbol{X}$ and $\Phi=\{\phi_{ij}\ |\ i,j\in[n], i\neq j\}$ is unknown, human knowledge can help to determine $\phi_{ij}<\phi_{lk}$ or not and generate $\mathcal{Q}$. The goal of ordinal embedding is to estimate $\boldsymbol{X}$ or $\Phi$ based on $\mathcal{Q}$.

One common class of ordinal embedding methods tries to formulate it as a classification problem (generally, a binary classification problem, say, \cite{agarwal2007generalized,tamuz2011adaptiive,vandermaaten2012stochastic,Terada2014LocalOE,amid2015multiview}). Given an ordered quadruplet $q=(i,j,l,k)$ and the associated ordered pair $(\phi_{ij}, \phi_{lk})$, the corresponding label $y_q$ can be defined as follows
\begin{equation}
  \label{eq:ordinal_label}
  y_{q}
  \left\{
  \begin{matrix}
  >0,&\phi_{ij}<\phi_{lk},\\
  <0,&\phi_{ij}>\phi_{lk}.
  \end{matrix}
  \right.
\end{equation}
Here we ignore the multi-class case, e.g. $y_q$ could be $\{-1, 0, +1\}$ and $y_q = 0$ indicates that $\phi_{ij}$ and $\phi_{lk}$ have the same value. As it is exceptionally rare in the practical applications and has no obvious improvement of the results whether we include multi-class label or not, we only consider the binary case in our generalized ordinal embedding (\textit{GOE}) problem.

Let $\mathcal{Y}_\mathcal{Q}:=\{y_q, q\in \mathcal{Q}\}$ be the corresponding label set.
Given an embedding candidate $\boldsymbol{X}$ and a classifier $h:\mathbb{R}_+\times\mathbb{R}_+\rightarrow\mathcal{Y}_\mathcal{Q}$, the empirical misclassification error can be defined as follows
\begin{equation}
  \label{eq:empirical_error}
  \mathcal{L}_{\mathcal{Q},\phi,h}(\boldsymbol{X}, \mathcal{Y}_\mathcal{Q}) = \frac{1}{|\mathcal{Q}|} \sum_{q\in\mathcal{Q}} \ell(h(\phi(\boldsymbol{x}_i,\boldsymbol{x}_j),\phi(\boldsymbol{x}_l,\boldsymbol{x}_k), y_q)),
\end{equation}
where $|\mathcal{Q}|$ represents the cardinality of the set $\mathcal{Q}$, and $\ell: \mathbb{R} \times \mathbb{R} \rightarrow \mathbb{R}_+\cup \{0\}$ is a specific loss function such as hinge loss or logistic loss.
Therefore, the \textit{GOE} problem can be formulated as the following minimization problem,
\begin{equation}
	\label{opt:nonconvex_goe}
	\underset{\boldsymbol{X}\in\mathbb{R}^{p\times n}}{\min}\ \ \mathcal{L}_{\mathcal{Q},\phi, h}(\boldsymbol{X}, \mathcal{Y}_\mathcal{Q}).
\end{equation}
In practice, the dissimilarity function $\phi$ is generally taken as the squared Euclidean distance $\phi(\boldsymbol{x}_i,\boldsymbol{x}_j)=d^2_{ij} = \|\boldsymbol{x}_i - \boldsymbol{x}_j\|_2^2$, and the empirical loss \eqref{eq:empirical_error} can be written as
\begin{equation}
	\label{eq:empirical_error_distance}
	\begin{aligned}
		& & &\ \ \mathcal{L}_{\mathcal{Q},\phi,h}(\boldsymbol{X}, \mathcal{Y}_\mathcal{Q})\ =\ \mathcal{L}_{\mathcal{Q},h}(\boldsymbol{D},\mathcal{Y}_\mathcal{Q})\\
		& &=&\ \ \frac{1}{|\mathcal{Q}|} \sum_{q\in\mathcal{Q}} \ell(h(d^2_{ij},d^2_{lk}, y_q)),
	\end{aligned}
\end{equation}
where $\boldsymbol{D}$ is the squared Euclidean distance matrix of $\boldsymbol{X}$.

Besides \eqref{opt:nonconvex_goe}, the following \textit{SDP} based formulation of the ordinal embedding is commonly used in the literature (\cite{agarwal2007generalized,tamuz2011adaptiive,vandermaaten2012stochastic}). Let $\boldsymbol{G} = \boldsymbol{X^\top X}$ be the Gram matrix of $\boldsymbol{X}$. There exists a bijection between the Gram matrix $\boldsymbol{G}\in\mathbb{S}^{n}_+$, $\mathbb{S}^{n}_+$ is the $n$-dimensional positive semi-definite cone, the set of all symmetric positive semidefinite matrices in $\mathbb{R}^{n\times n}$ and the squared Euclidean distance matrix $\boldsymbol{D}$ as $d^2_{ij} \ =\ \|\boldsymbol{x}_i - \boldsymbol{x}_j\|_2^2\ =\ g_{ii}-2g_{ij}+g_{jj}$, where $g_{ij}$ is the $(i,j)$ element of $\boldsymbol{G}$. 

We change the variable $\boldsymbol{D}$ in empirical loss \eqref{eq:empirical_error_distance} as $\boldsymbol{G}$
\begin{equation}
	\label{eq:empirical_error_gram}
	\begin{aligned}
		& & &\ \ \mathcal{L}_{\mathcal{Q},h}(\boldsymbol{D},\mathcal{Y}_\mathcal{Q})\ =\ \mathcal{L}_{\mathcal{Q},h}(\boldsymbol{G},\mathcal{Y}_\mathcal{Q})\\
		& &=&\ \ \frac{1}{|\mathcal{Q}|} \sum_{q\in\mathcal{Q}} \ell(h(g_{ii}-2g_{ij}+g_{jj},g_{ll}-2g_{lk}+g_{kk}, y_q)).	
	\end{aligned}
\end{equation}
According to \eqref{opt:nonconvex_goe}, \eqref{eq:empirical_error_distance} and \eqref{eq:empirical_error_gram}, the ordinal embedding problem can be formulated as the following \textit{SDP} problem with respect to $\boldsymbol{G}$, i.e.,
\begin{equation}{}
	\label{opt:convex_goe}
	\underset{\boldsymbol{G}\in\mathbb{S}^{n}_+,\ \textit{rank}(\boldsymbol{G})\leq p}{\min}\ \ \mathcal{L}_{\mathcal{Q},h}(\boldsymbol{G},\mathcal{Y}_\mathcal{Q}),
\end{equation}
the positive semi-definite constraint $\boldsymbol{G}\in\mathbb{S}^{n}_+$ or $\boldsymbol{G}\succeq 0$ comes from the fact that the Gram matrix $\boldsymbol{G}$ is positive semi-definite matrix; the rank constraint comes from the fact that $\textit{rank}(\boldsymbol{G})\leq \textit{rank}(\boldsymbol{X})\leq \min(n,p)=p$. Note that the formulation \eqref{opt:convex_goe} is generally convex. However, the computational complexity of such \textit{SDP} problem is very high, which degrades the scalability of this kind of methods. This motivates us to directly obtain embedding $\boldsymbol{X}$ from \eqref{opt:nonconvex_goe}.

\section{Development of SVRG-SBB}

Since \eqref{opt:nonconvex_goe} is an unconstrained optimization problem, \textit{SVD} and regularization parameter tuning are both avoided. However, without any prior knowledge on $\mathcal{O}$, the sample complexity of $\mathcal{Q}$ is $\boldsymbol{O}(n^4)$. Because of the expense of full gradients and inverse of Hessian matrix computation in each iteration, the traditional full batch optimization methods, i.e. gradient descent and (quasi-)Newton method, are not suitable for solving such large-scale problem where $n$ would be larger than thousands.
Instead of the full-batch methods, we introduce the stochastic algorithm to solve the non-convex problem \eqref{opt:nonconvex_goe}. One open issue in stochastic optimization is how to choose an appropriate step size in practice. Traditional methods include that using a constant step size to track the iterations, adopting a diminishing step size to enforce convergence, or tuning a step size empirically which can be time-consuming. Recently, \cite{NIPS2016_6286} proposed to use the Barzilai-Borwein (\textit{BB}) method to automatically compute step sizes in \textit{SVRG} for strongly convex objective function. Their method is called ``\textit{SVRG-BB}''. In the following part, we will analyze the existing problem of \textit{BB} method when it is adopted in the non-convex problem. Furthermore, we propose the stabilized Barzilai-Borwein (\textit{SBB}) step size which alleviates these issues and establish the non-asymptotic convergence analysis of the proposed ``\textit{SVRG-SBB}'' algorithm.

\subsection{The Existing Problems of Barzilai-Borwein Method}

In machine learning and data mining, we often encounter the unconstrained minimization problem (\ref{opt:nonconvex_goe}) as a finite-sum problem. Let $f_1,\dots,f_n$ be a sequence of vector function as $f_i:\mathbb{R}^{p}\rightarrow\mathbb{R}$, and our goal is to obtain an approximation solution of the following finite-sum problem
\begin{equation}
  \label{eq:finit-sum}
  \underset{\boldsymbol{x}}{\min}\ f(\boldsymbol{x})=\frac{1}{n}\sum^n_{i=1}f_i(\boldsymbol{x}),
\end{equation}
where $n$ is the training sample size, and each $f_i$ is the cost function or loss function corresponding to the $i^{\text{th}}$ training sample. Regardless the convexity of $f$, the choice of step size in stochastic optimization always depends on the Lipschitz constant of $f$, which is usually difficult to estimate in practice. \textit{BB} step size has been incorporated with \textit{SVRG} in \cite{NIPS2016_6286} but it is restricted to the case where their assumptions are each $f_i$ is convex, differentiable and $f$ is strongly convex. This assumptions are adopted due to the use of a strongly-convex regularization such as the squared $\ell_2$-norm. However, there are many important large-scale non-convex optimization problems, such as neural network.

The original \textit{BB} method, proposed by Barzilai and Borwein in \cite{barzilai1988two}, has been proven to be very effective in solving nonlinear optimization problems via gradient descent. One possible choice of \textit{BB} step size $\eta_t$ is
\begin{equation}
    \label{eq:original_bb_stepsize}
    \eta_t = \frac{\|\triangle{\boldsymbol{x}}_t\|^2}{\triangle{\boldsymbol{x}}_t^\top\triangle{\boldsymbol{y}}_t}
\end{equation}
where $\triangle{\boldsymbol{x}}_t = {\boldsymbol{x}}_t-{\boldsymbol{x}}_{t-1}$ and $\triangle{\boldsymbol{y}}_t = \nabla f({\boldsymbol{x}}_t)-\nabla f({\boldsymbol{x}}_{t-1})$. Actually \textit{BB} step size is a possible solution of the so-called ``secant equation''.

As the original \textit{BB} method approximates the inverse of Hessian matrix of $f$ at $\boldsymbol{x}_t$ by $\frac{1}{\eta_t}\boldsymbol{I}$, there exist some inherent drawbacks toward extending the step size \eqref{eq:original_bb_stepsize} to non-convex optimization problems. If $f$ is differentiable and $\mu$-strongly convex, it holds that
\begin{equation}
    \triangle{\boldsymbol{x}}_t^\top\triangle{\boldsymbol{y}}_t\geq\mu\|\triangle {\boldsymbol{x}}_t\|^2>0
\end{equation}
which implies $\eta_t$ is always positive. However, if $f$ is differentiable and convex, we have
\begin{equation}
	\label{eq:curvature}
    \triangle{\boldsymbol{x}}_t^\top\triangle{\boldsymbol{y}}_t\geq0
\end{equation}
and \eqref{eq:original_bb_stepsize} might approach $\infty$ when the denominator of \eqref{eq:original_bb_stepsize} is extremely small.
Furthermore, if the differentiable function $f$ is non-convex,
the denominator of \eqref{eq:original_bb_stepsize} might even be negative that makes \textit{BB} method fail.

\begin{example}
	given a quadratic optimization problem
\begin{equation}
    \label{eq:quadric_opt}
    \underset{{\boldsymbol{x}}\in\mathbb{R}^{p}}{\min} f({\boldsymbol{x}}) := \frac{1}{2}\boldsymbol{x}^\top\boldsymbol{A} {\boldsymbol{x}}
\end{equation}
where $\boldsymbol{A}$ is a diagonal matrix with $\{\lambda_i\}^p_{i=1}$ as the diagonal entries, we set $p=3, \lambda_1=1, \lambda_2=0, \lambda_3=-1 $. The initial ${\boldsymbol{x}}_0=[0, 0, 1]^\top$ and by gradient descent with \textit{BB} step size, ${\boldsymbol{x}}_1 = {\boldsymbol{x}}_0-\eta_0\boldsymbol{A} {\boldsymbol{x}}_0 = [0, 0, 1+\eta_0]$. The corresponding value of \eqref{eq:original_bb_stepsize} is $-1$. If the initial ${\boldsymbol{x}}_0=[0, 1, 0]^\top$, the denominator of \eqref{eq:original_bb_stepsize} is $0$.
\end{example}

\subsection{SVRG-SBB for Ordinal Embedding}
An intuitive way to overcome the flaw of \textit{BB} step size is to keep the denominator of \eqref{eq:original_bb_stepsize} positive and control the lower bound of $\triangle\boldsymbol{x}_t^\top\triangle\boldsymbol{y}_t$ in each iteration, which leads to our proposed stabilized Barzilai-Borwein (\textit{SBB}) step size shown as follows
\begin{equation}
    \label{eq:sbb_step}
    \eta_{\epsilon,t} =
    \frac{\|\triangle{\boldsymbol{x}}_t\|^2}{|\triangle{\boldsymbol{x}}_t^\top\triangle{\boldsymbol{y}}_t|+\epsilon\|\triangle{\boldsymbol{x}}_t\|^2}.
\end{equation}
Actually, as shown by our latter convergence result, if the Hessian of the objective function $\nabla^2 f(\boldsymbol{x})$ is nonsingular and the magnitudes of its eigenvalues are lower bounded by some positive constant $\mu$, then we can take $\epsilon=0$. In this case, we call the referred step size \textit{SBB}$_0$ henceforth. Even if we have no information about the Hessian of the objective function in practice, the \textit{SBB}$_\epsilon$ step size with an $\epsilon>0$ is just a more conservative version of \textit{SBB}$_0$ step size.

From \eqref{eq:sbb_step}, if the gradient $\nabla f_i$ is Lipschitz continuous with constant $L>0$ (i.e., $\|\nabla f_i ({\boldsymbol{x}}) - \nabla f_i({\boldsymbol{y}})\| \leq L \|{\boldsymbol{x}} - {\boldsymbol{y}}\|, \forall\ {\boldsymbol{x}}, {\boldsymbol{y}} \in \mathbb{R}^p$), then the \textit{SBB}$_\epsilon$ step size can be bounded as follows
\begin{align}
    \label{eq:bound-sbb}
    \frac{1}{L+\epsilon} \leq \eta_{\epsilon,t} \leq \frac{1}{\epsilon},
\end{align}
where the lower bound is obtained by the $L$-Lipschitz continuity of $\nabla f$, and the upper bound is directly derived by its specific form. Furthermore, if $\nabla^2 f(\boldsymbol{x})$ is nonsingular and its eigenvalues have a lower bound $\mu>0$, the bound of \textit{SBB}$_0$ becomes
\begin{align}
    \label{eq:bound-sbb0}
    \frac{1}{L} \leq \eta_{0,t} \leq \frac{1}{\mu}.
\end{align}

The proposed \textit{SVRG-SBB} algorithm is described in \textbf{Algorithm} \ref{alg:svrg-sbb}. As shown in Figure \ref{fig:step}, \textit{SBB}$_\epsilon$ step size with a positive $\epsilon$ can make it more stable when \textit{SBB}$_0$ step size is unstable and varies dramatically. Moreover, \textit{SBB}$_\epsilon$ step size usually changes significantly only at the initial several epochs, and then quickly gets very stable. This is mainly because there are many iterations in an epoch of \textit{SVRG-SBB}, and the algorithm might close to a stationary point after only a few epochs. After the initial epochs, the \textit{SBB}$_\epsilon$ step sizes might be very close to the inverse of the objective function curvature as $\epsilon$ is small enough.

The main difference between \textbf{Algorithm} \ref{alg:svrg-sbb} and the former version is that we adopt mini-batch in the inner loop. Mini-batching is a useful strategy for large-scale optimization problem, especially in multi-core and distributed settings as it greatly helps one exploiting parallelism and reducing the communication burden. When the mini-batch size is $1$, \textbf{Algorithm} \ref{alg:svrg-sbb} reduces to our former algorithm. To incorporate mini-batches, we replace single sample gradient update with sampling (with replacement) a subset $\mathcal{I}_t\subset[n]$ with its cardinality $|{\cal I}_t|=b$, where $b \in \mathbb{N}$ is the mini-batch size. We update the $\boldsymbol{x}^{s}_{t}$ by the modified \textit{SBB} step size as ${\boldsymbol{x}}^{s}_{t+1}={\boldsymbol{x}}^{s}_{t} - b\eta_{\epsilon,s-1} {\boldsymbol{u}}^{s}_{t}$, where $\eta_{t,s-1}$ is specified in \eqref{Eq:eta-epsilon-s} and $\boldsymbol{u}^{s}_{t}$ is defines as
\begin{equation*}
    {\boldsymbol{u}}^{s}_{t} = \frac{1}{b}\underset{i_t\in\mathcal{I}_t}{\sum}\left(\nabla f_{i_t}({\boldsymbol{x}}^{s}_t) - \nabla f_{i_t}(\tilde{{\boldsymbol{x}}}^{s})\right)+\nabla f(\tilde{\boldsymbol{x}}^s).
\end{equation*}

\begin{algorithm}
    \small
    \caption{SVRG-SBB for \eqref{eq:finit-sum}}\label{alg:svrg-sbb}
    \begin{algorithmic}
        \REQUIRE
        $\epsilon\geq0$, update frequency $m$, mini-batch size $b$, maximal number of iterations $S$, initial step size $\eta_0$ (only used in the first epoch), initial point $\tilde{{\boldsymbol{x}}}^0 \in \mathbb{R}^{p}$, and the number of training samples $n$
        \FOR{$ s=0$ to $ S-1$}
        \STATE
          \begin{equation*}
            \boldsymbol{g}^{s+1} = \nabla f(\tilde{{\boldsymbol{x}}}^s) =\frac{1}{n}\sum^n_{i=1}\nabla f_i(\tilde{{\boldsymbol{x}}}^s)
          \end{equation*}
        \IF {$ s>0$}
        \STATE
              \begin{align}
              & \triangle {\boldsymbol{x}}_s = \tilde{{\boldsymbol{x}}}^{s}-\tilde{{\boldsymbol{x}}}^{s-1}, \
               \triangle{\boldsymbol{x}}_s =\ \ \ {\boldsymbol{g}}^{s+1}-{\boldsymbol{g}}^{s},\nonumber\\
              & \eta_{\epsilon,s} = \frac{1}{m}\cdot\frac{\|\triangle {\boldsymbol{x}}_s\|^2}{|\langle \triangle {\boldsymbol{x}}_s, \triangle\boldsymbol{y}_s \rangle|+\epsilon\|\triangle {\boldsymbol{x}}_s\|^2} \label{Eq:eta-epsilon-s}
              \end{align}
        \ENDIF
        \STATE ${\boldsymbol{x}}^{s+1}_0=\tilde{{\boldsymbol{x}}}^s$
        \FOR{$  t=0$ to $ m-1$}
        \STATE uniformly pick $ \mathcal{I}_t\subset[n]$ with $|\mathcal{I}_t|=b$
        \STATE ${\boldsymbol{u}}^{s+1}_t = \frac{1}{b}\underset{i_t\in\mathcal{I}_t}{\sum}\left(\nabla f_{i_t}({\boldsymbol{x}}^{s+1}_{t})-\nabla f_{i_t}(\tilde{\boldsymbol{x}}^s)\right)+\boldsymbol{g}^{s+1}$
        \STATE ${\boldsymbol{x}}^{s+1}_{t+1}= {\boldsymbol{x}}^{s+1}_{t}-b \eta_{\epsilon,s}{\boldsymbol{x}}^{s+1}_t$
        \ENDFOR
        \STATE $\tilde{\boldsymbol{x}}^{s+1}={\boldsymbol{x}}^{s+1}_{m}$
        \ENDFOR
        \ENSURE ${ \boldsymbol{x}_{\mathrm{out}}}$ is chosen uniformly from $\{\{\boldsymbol{x}^{s}_{t}\}_{t=1}^{m}\}_{s=1}^{S}$.
    \end{algorithmic}
\end{algorithm}

\section{Convergence Analysis of \textit{SVRG-SBB}}
\label{sc:convergence}

In this section, we first establish a sublinear rate of convergence (to a stationary point) of the proposed \textit{SVRG-SBB} under the mild smoothness condition, and then show the linear rate of convergence (i.e., converging exponentially fast to a global optimum) of its modular version under the furthered Polyak-{\L}ojasiewicz (PL) property \cite{Polyak1963}.

\subsection{Sublinear convergence under smoothness}
\label{sc:sublinear-smooth}

Throughout this section, we assume that each $f_i$ in \eqref{eq:finit-sum} is $L$-smoothness with a constant $L>0$, i.e., $f_i$ is continuously differentiable and its gradient $\nabla f_i$ is $L$-Lipschitz, shown as follows:
\begin{equation}
	\label{definition:lsmooth}
	\|\nabla f_i(\boldsymbol{x})-\nabla f_i(\boldsymbol{y})\|\leq L\|\boldsymbol{x}-\boldsymbol{y}\|,\ \forall\ \boldsymbol{x},\boldsymbol{y}\in\mathbb{R}^p.
\end{equation}
The $L$-smoothness assumption is very general to derive the convergence of an algorithm in literature (say, \cite{nesterov2013introductory}, \cite{pmlr-v48-reddi16}, \cite{Zeng2018-aistats}, \cite{Zeng2019}).
As shown in the supplementary material, all the loss functions adopted in the experiments for \textit{GOE} problem are verified to satisfy this assumption.

In the following, we provide a key lemma that illustrates the convergence behavior of the inner loop of \textbf{Algorithm \ref{alg:svrg-sbb}}. To state this lemma, we first define several positive constants and sequences which are used in our analysis. Given a positive sequence $\{\beta_s\}_{s=0}^{S-1}$, for any $0\leq s \leq S-1$, we define
\begin{align}
&\rho_s := 1+b \eta_{\epsilon,s}\beta_s + 2b\eta_{\epsilon,s}^2L^2, \label{eq:rhos}
\end{align}
and for any $t=0,\ldots,m-1$,
\begin{align}
\label{eq:ct}
c_{t,s} := \rho_s c_{t+1,s} + b \eta_{\epsilon,s}^2L^3
\end{align}
with ${c}_{m,s}=0$,
and
\begin{align}
\label{eq:Gammas}
\Gamma_{t,s} :=
b \eta_{\epsilon,s} \left[1-c_{t+1,s}\beta_s^{-1} -b\eta_{\epsilon,s}(L+2c_{t+1,s})\right].
\end{align}
Based on these sequences, we present the lemma as follows.

\begin{lemma}
\label{keylemma-minibatch}
Suppose that each $f_i$ is $L$-smoothness (i.e., satisfying \eqref{definition:lsmooth}).
Let $\{{\boldsymbol{x}}_t^{s+1}\}_{t=1}^m$ be a sequence generated by \textbf{Algorithm \ref{alg:svrg-sbb}} at the $s^{\text{th}}$ inner loop, $s=0,\ldots, S-1$. Let $b, m, \epsilon$ and $\beta_s$  be chosen such that
\begin{align}
\label{Eq:cond-gamma}
b\eta_{\epsilon,s}(L+2c_{t+1,s}) + c_{t+1,s}\beta_s^{-1} <1, \  t=0,\ldots,m-1,
\end{align}
then
\[
  \mathbb{E}[\|\nabla f({\boldsymbol{x}}_t^{s+1})\|^2] \leq \frac{{R}_{t,s+1} - {R}_{t+1,s+1}}{{\Gamma}_{t,s}},
\]
where
${R}_{t,s+1}:= \mathbb{E}[f({\boldsymbol{x}}_t^{s+1}) + {c}_{t,s} \|{\boldsymbol{x}}_t^{s+1} - \tilde{\boldsymbol{x}}^s\|^2].$
\end{lemma}
Lemma \ref{keylemma-minibatch} shows that the inner loop of SVRG-SBB would decrease along the defined Lyapunov function $R_{t,s}$, which contains the function value sequence itself as well as a proximal term $\|{\boldsymbol{x}}_t^{s+1} - \tilde{\boldsymbol{x}}^s\|^2$. Particularly, $R_{0,s+1} = \mathbb{E}[f(\tilde{\boldsymbol{x}}^{s})]$ and $R_{m,s+1} = \mathbb{E}[f(\tilde{\boldsymbol{x}}^{s+1})]$. This lemma indicates that
\begin{align*}
\mathbb{E}[f(\tilde{\boldsymbol{x}}^{s})] - \mathbb{E}[f(\tilde{\boldsymbol{x}}^{s+1})] \geq \sum_{t=0}^{m-1} \Gamma_{t,s} \mathbb{E}[\|{\boldsymbol{x}}_t^{s+1}\|^2] \geq 0.
\end{align*}
If $f$ is lower bounded (say, lower bounded by $0$), the above inequality demonstrates that the function value sequence $\{f(\tilde{\boldsymbol{x}}^s)\}$ converges in expectation. We present its proof latter in Appendix A for the completeness.

Lemma \ref{keylemma-minibatch} implies the decreasing property and thus functional value convergence of the outer loop sequence generated by Algorithm \ref{alg:svrg-sbb}.
In the next, we provide an abstract result on the convergence rate of Algorithm \ref{alg:svrg-sbb}, whose proof is presented latter in Appendix B.

\begin{theorem}
\label{Thm:intermediate-result}
Suppose that each $f_i$ is $L$-smoothness (i.e., satisfying \eqref{definition:lsmooth}).
Let $\{\boldsymbol{x}_t^s\} (s=0,\ldots S-1, t=0,\ldots,m-1)$ be a sequence generated by Algorithm \ref{alg:svrg-sbb}.
Let $b, m, \epsilon$ and $\{\beta_s\}_{s=0}^S$  be chosen such that \eqref{Eq:cond-gamma} holds for any $s=0,\ldots,S-1$ and $t=0,1,\ldots,m-1$.
There holds
\begin{align*}
\mathbb{E}[\|\nabla f(\boldsymbol{x}_{\mathrm{out}})\|^2] \leq \frac{f(\tilde{\boldsymbol{x}}^0)-f({\boldsymbol{x}}^*)}{T \gamma_S},
\end{align*}
where $T= m  S$ is the total number iterations, $\boldsymbol{x}^*$ is an optimal solution to \eqref{eq:finit-sum}, and $\gamma_S = \min_{0\leq s \leq S-1} \min_{0\leq t\leq m-1} \Gamma_{t,s}$.
\end{theorem}

This theorem shows the ${\bf O}(1/T)$  convergence rate of SVRG-SBB method under certain conditions. Such rate is consistent with that of the mini-batch SVRG method studied in \cite[Theorem 2]{pmlr-v48-reddi16} and faster than the convergence rate of SGD as ${\bf O}(1/\sqrt{T})$ established in \cite{Nemirovski1983,pmlr-v48-reddi16}.
Note that the condition \eqref{Eq:cond-gamma} is rather technical.
In the following, we provide several sufficient conditions of \eqref{Eq:cond-gamma}.
Specifically, we take
\begin{align}
\label{Eq:betas}
\beta_s = 4mb\eta_{\epsilon,s}^2L^3, \ s=0,\ldots,S-1.
\end{align}

\begin{theorem}
    \label{svrg_sbb_minibatch}
    Let $\{\{{\boldsymbol{x}}_t^s\}_{t=1}^{m}\}_{s=1}^S$ be a sequence generated by Algorithm \ref{alg:svrg-sbb}.
    Suppose that each $f_i$ is $L$-smoothness. For any given $\epsilon>0$, if
    \begin{align}
    \label{Eq:cond-m-b}
    b <\min\left\{ \frac{m\epsilon^2}{L^2\left(1+\sqrt{1+4\epsilon/L}\right)}, \frac{m\epsilon}{2L\left(1 + \sqrt{1+4L/\epsilon}\right)}\right\},
    \end{align}
    then for the output ${\boldsymbol{x}}_{\mathrm{out}}$ of Algorithm \ref{alg:svrg-sbb}, we have
    \begin{align}
        \label{eq:mini-rate}
        \mathbb{E}[\|\nabla f({\boldsymbol{x}}_{\mathrm{out}})\|^2] \leq \frac{4(f(\tilde{\boldsymbol{x}}^0)-f({\boldsymbol{x}}^*))}{T b \eta_{\epsilon,\min}},
    \end{align}
    where $\eta_{\epsilon,\min}:=\min_{0\leq s\leq S-1}\eta_{\epsilon,s}$, and $\boldsymbol{x}^*$ is a global minimum.
\end{theorem}

The proof of Theorem \ref{svrg_sbb_minibatch} is presented in Appendix C.
Given an $\epsilon>0$ and an update frequency $m$ (commonly taken as the multiple times of the total sample size $n$),
Theorem \ref{svrg_sbb_minibatch} shows that the proposed SVRG method converges to a stationary point at a sublinear rate if the objective function is $L$-smoothness and the mini-batch size $b$ is less than an upper bound. When $b=1$, the mini-batch version of SVRG-SBB reduces to the stochastic version of SVRG-SBB. By \eqref{eq:mini-rate}, the established rate of SVRG-SBB is the same with that of SVRG studied in \cite{pmlr-v48-reddi16}. By \eqref{eq:mini-rate} again, a larger $b$ adopted in the inner loop generally implies faster convergence, which is also verified by our latter experiments. Note that the upper bound of $b$ in \eqref{Eq:cond-m-b} is related to $m$ and the ratio $\frac{\epsilon}{L}$, and it is monotone increasing with respect to both $m$ and $\frac{\epsilon}{L}$. Particularly, when $\frac{\epsilon}{L}=1$, then the upper bound of $b$ in \eqref{Eq:cond-m-b} is $\frac{m}{2(1+\sqrt{5})}$,
and if $m$ is further taken as the multiple times of sample size $n$ (say, $m = 2(1+\sqrt{5}) n$), in this case, the upper bound of $b$ is the sample size $n$.
This implies that the choice of mini-batch size $b$ is generally very flexible when $\frac{\epsilon}{L}\approx 1$.
Moreover, if the curvature of the objective function is  lower bounded by some $\mu>0$, then according to the proof of Theorem \ref{svrg_sbb_minibatch} (in Appendix C), the parameter $\epsilon$ emerging in the upper bound of $b$ should be replaced by $\epsilon' = \mu +\epsilon$,
while when $\mu$ is moderately large, the parameter $\epsilon$ can be even  taken as $0$, and in this case, the upper bound of $b$ becomes
\[
b < \min \left\{\frac{m}{\kappa^2(1+\sqrt{1+4\kappa^{-1}})}, \frac{m}{2\kappa(1+\sqrt{1+4\kappa})} \right\},
\]
where $\kappa:= \frac{L}{\mu}$ represents the \textit{condition number} of the objective function.
Formally, we state these claims in the following corollary.

\begin{corollary}
\label{coro:sbb_0}
Under the conditions of Theorem \ref{svrg_sbb_minibatch}, if the Hessian $\nabla^2 f(\boldsymbol{x})$ exists and $\mu$ is the lower bound of the magnitudes of eigenvalues of $\nabla^2 f(\boldsymbol{x})$ for any bounded $\boldsymbol{x}$, the convergence rate \eqref{eq:mini-rate} still holds for Algorithm \ref{alg:svrg-sbb} with $\epsilon$ replaced by $\mu+\epsilon$. In addition, if $\mu>0$, then we can take $\epsilon=0$, and \eqref{eq:mini-rate} still holds for SVRR-SBB$_0$ with $\epsilon$ replaced by $\mu$.
\end{corollary}


\subsection{Linear Convergence under PL Property}
\label{sc:linear-PL}
\begin{algorithm}
    \small
    \caption{SVRG-SBB for \eqref{eq:finit-sum} with PL Property}\label{alg:svrg-sbb1}
    \begin{algorithmic}
        \REQUIRE
        $\epsilon\geq0$, update frequency $m$, mini-batch size $b$, maximal number of iterations $S$ in every SVRG-SBB module,  initial point $\tilde{\boldsymbol{x}}^0 \in \mathbb{R}^{p}$, and the number of modules $K$
        \FOR{$ k=1$ to $ K$}
        \STATE $\tilde{\boldsymbol{x}}^k = \text{SVRG-SBB}(\tilde{\boldsymbol{x}}^{k-1},S,m,\{\eta_{\epsilon,s}\}_{s=0}^{S-1})$
        \ENDFOR
        \ENSURE $\tilde{\boldsymbol{x}}^K$
    \end{algorithmic}
\end{algorithm}

In the next, we consider the convergence of a modular version of the proposed SVRG-SBB method, that is, let every $S$ iterates of SVRG-SBB be one module, then the output is yielded after running several modules as described in Algorithm \ref{alg:svrg-sbb1}.
Note that the computational complexity of running $K$ modules in Algorithm \ref{alg:svrg-sbb1} is the same as that of running $K \times S$ iterations of the original SVRG-SBB, i.e., Algorithm \ref{alg:svrg-sbb}.
However, by exploiting such modification, we can show latter the linear convergence of Algorithm \ref{alg:svrg-sbb1} under the furthered Polyak-{\L}ojasiewicz (PL) property \cite{Polyak1963}, whose definition is stated as follows.

\begin{definition}
  A continuously differentiable function $f: \mathcal{X}\rightarrow\mathbb{R}$ is said to be $\lambda$-Polyak-{\L}ojasiewicz with some constant $\lambda>0$, if it satisfies the following PL inequality
  \begin{equation}
    \label{eq:pl_fun}
    \frac{1}{2}\|\nabla f({\boldsymbol{x}})\|^2\geq\lambda(f({\boldsymbol{x}})-f({\boldsymbol{x}}^*)),\ \forall\ {\boldsymbol{x}}\in\mathcal{X}
  \end{equation}
  where ${\boldsymbol{x}}^*=\underset{{\boldsymbol{x}}\in\mathcal{X}}{\arg\min}\ f(\boldsymbol{x})$.
\end{definition}

The PL inequality is widely used to derive the linear convergence of the existing methods in literature (say, \cite{Karimi2016,pmlr-v48-reddi16}). 
Some typical examples satisfying PL inequality include the strongly convex function, the square loss and logistic loss commonly used in machine learning, and some invex functions like $f(x)=x^2+3\sin^2(x)$ \cite{Karimi2016}.
According to \cite{Karimi2016}, the function $f(x)=x^2+3\sin^2(x)$ is nonconvex and satisfies the PL inequality with $\lambda = \frac{1}{32}$.
For more examples, we refer to \cite{Karimi2016} and references therein.


\begin{theorem}
    \label{svrg_sbb_pl}
    Let $\{\tilde{\boldsymbol{x}}^k\}_{k=1}^K$ be a sequence generated by Algorithm \ref{alg:svrg-sbb1}.
    Suppose that assumptions in Theorem \ref{svrg_sbb_minibatch} hold, and further that $f$ satisfies the Polyak-\L{}ojasiewicz inequality \eqref{eq:pl_fun} for some $\lambda>0$.
    If $\frac{1}{2}\lambda m S b \eta_{\epsilon,\min}  \geq \rho^{-1}$ for some $\rho \in (0,1)$, then we have
    \begin{align}
        \label{eq:pl-inner-rate}
        \mathbb{E}[\|\nabla f(\tilde{\boldsymbol{x}}^{k})\|^2] \leq {\rho^{-k}} \|\nabla f({\tilde{\boldsymbol{x}}^{0}})\|^2,
    \end{align}
    and
    \begin{align}
        \label{eq:pl-global}
        \mathbb{E}[f(\tilde{\boldsymbol{x}}^{k})-f(\boldsymbol{x}^{*})] \leq {\rho^{-k}}(f(\tilde{\boldsymbol{x}}^{0})-f(\boldsymbol{x}^{*})).
    \end{align}
\end{theorem}

Theorem \ref{svrg_sbb_pl} shows that Algorithm \ref{alg:svrg-sbb1} converges exponentially fast to a global optimum if the objective function satisfies the PL inequality. The proof of this theorem is presented in Appendix D.

\section{Experiments}
\label{section:experiment}
In this section, we conduct a series of comprehensive experiments on synthetic data and real-world data to demonstrate the effectiveness of the proposed algorithms. Four objective functions including \textit{GNMDS}\cite{agarwal2007generalized}, \textit{CKL}\cite{tamuz2011adaptiive}, \textit{STE} and \textit{TSTE}\cite{vandermaaten2012stochastic} are taken into consideration. We notice that some new methods are proposed for ordinal embedding, such as \textit{SOE/LOE}\cite{Terada2014LocalOE} and \textit{MVE}\cite{amid2015multiview}. However, our main contribution is the optimization algorithm, \textit{SVRG-SBB}$_\epsilon$ and its mini-batch variant, which can also be applied to solve \textit{SOE/LOE} and \textit{MVE}. Moreover, the objective functions of \textit{SOE/LOE} and \textit{MVE} are very similar to \textit{GNMDS} and \textit{STE}. We compare the performance of stochastic methods (including \textit{SGD}, \textit{SVRG}, \textit{SVRG-SBB}$_0$, \textit{SVRG-SBB}$_\epsilon$, and \textit{SVRG-SBB}$_\epsilon$ \textit{mini-batch}) with that of deterministic method (projection gradient descent) for solving convex and non-convex formulations of these four objective functions.

\subsection{Simulations}

\begin{table*}
  \centering
  \begin{subtable}[t]{0.45\textwidth}
    \caption{GNMDS}
    \begin{tabular}{l|cccc}
      \toprule
        method                          & min       & mean     & max      & std       \\ \midrule
        cvx                             & -         & -        & -        & -         \\
        ncvx Batch                      & 4.3760    & 4.7466   & 5.4570   & 0.2966    \\
        ncvx SGD                        & -         & -        & -        & -         \\
        ncvx SVRG                       & 6.5280    & 7.9204   & 9.5780   & 0.8233    \\
        ncvx SVRG-SBB$_0$               & {0.5120}  & {0.6398} & {0.8360} & {0.0719}  \\
        ncvx SVRG-SBB$_\epsilon$-$1$    & {0.7260}  & {0.9119} & {1.1550} & {0.1024}  \\
        ncvx SVRG-SBB$_\epsilon$-$5$    & {0.4210}  & {0.5298} & {0.6720} & {0.0615}  \\
        ncvx SVRG-SBB$_\epsilon$-$10$   & {0.4010}  & {0.4810} & {0.6140} & {0.0553}  \\
        ncvx SVRG-SBB$_\epsilon$-$20$   & \textbf{0.3800}  & \textbf{0.4581} & \textbf{0.5730} & \textbf{0.0535}  \\
        ncvx SVRG-SBB$_\epsilon$-$50$   & {0.4423}  & {0.5162} & {0.6427} & {0.0548}  \\
        ncvx SVRG-SBB$_\epsilon$-$100$  & {0.7657}  & {1.0431} & {1.3730} & {0.1405}  \\
      \bottomrule
    \end{tabular}
    \label{tab:subtable_gnmds}
  \end{subtable}
  \vspace{\fill}
  \begin{subtable}[t]{0.45\textwidth}
    \caption{CKL}
    \begin{tabular}{l|cccc}
      \toprule
        method                          & min       & mean     & max      & std       \\ \midrule
        cvx                             & -         & -        & -        & -         \\
        ncvx Batch                      & 2.4600    & 2.5075   & 2.6490   & 0.0346    \\
        ncvx SGD                        & 1.8360    & 2.4086   & 3.4120   & 0.3210    \\
        ncvx SVRG                       & 2.0620    & 2.4075   & 2.9720   & 0.1910    \\
        ncvx SVRG-SBB$_0$               & 0.5130    & 0.7010   & 1.1740   & 0.1183    \\
        ncvx SVRG-SBB$_\epsilon$-$1$    & 1.0180    & 1.1720   & 1.4290   & 0.0929    \\
        ncvx SVRG-SBB$_\epsilon$-$5$    & 0.6130    & 0.7093   & 0.8680   & 0.0512    \\
        ncvx SVRG-SBB$_\epsilon$-$10$   & 0.5490    & 0.6484   & 0.7920   & 0.0499    \\
        ncvx SVRG-SBB$_\epsilon$-$20$   & 0.5250    & 0.6176   & 0.7560   & 0.0490    \\
        ncvx SVRG-SBB$_\epsilon$-$100$  & \textbf{0.5172}    & \textbf{0.6013}   & \textbf{0.7433}   & \textbf{0.0478}    \\
        ncvx SVRG-SBB$_\epsilon$-$200$  & 1.1380    & 1.3083   & 1.7200   & 0.1120    \\
      \bottomrule
    \end{tabular}
    \label{tab:subtable_ckl}
  \end{subtable}

  \hspace{\fill}

  \begin{subtable}[t]{0.45\textwidth}
    \caption{STE}
    \begin{tabular}{l|cccc}
      \toprule
        method                          & min       & mean     & max      & std       \\ \midrule
        cvx                             & -         & -        & -        & -         \\
        ncvx Batch                      & 3.4520    & 3.5765   & 3.7310   & 0.0676    \\
        ncvx SGD                        & 5.8640    & 6.6043   & 6.7690   & 0.2123    \\
        ncvx SVRG                       & 2.7930    & 3.2328   & 3.9710   & 0.2521    \\
        ncvx SVRG-SBB$_0$               & 0.4580    & 0.6644   & 0.8630   & 0.0880    \\
        ncvx SVRG-SBB$_\epsilon$-$1$    & 0.9100    & 1.0656   & 1.3040   & 0.0803    \\
        ncvx SVRG-SBB$_\epsilon$-$5$    & 0.5350    & 0.6354   & 0.7700   & 0.0492    \\
        ncvx SVRG-SBB$_\epsilon$-$10$   & 0.4920    & 0.5814   & 0.7340   & 0.0511    \\
        ncvx SVRG-SBB$_\epsilon$-$20$   & 0.4610    & 0.5511   & 0.6740   & 0.0447    \\
        ncvx SVRG-SBB$_\epsilon$-$100$  & \textbf{0.4600} & \textbf{0.5414}& \textbf{0.6618}& \textbf{0.0501} \\
        ncvx SVRG-SBB$_\epsilon$-$200$  & 0.9259    & 1.1346   & 1.3449   & 0.0854    \\
      \bottomrule
    \end{tabular}
    \label{tab:subtable_ste}
  \end{subtable}
  \vspace{\fill}
  \begin{subtable}[t]{0.45\textwidth}
    \caption{TSTE}
    \begin{tabular}{l|cccc}
      \toprule
        method                          & min       & mean     & max      & std       \\
        \midrule
        kernel                          & -         & -        & -        & -         \\
        ncvx Batch                      & 6.3860    & 6.9228   & 7.6280   & 0.4334    \\
        ncvx SGD                        & 5.7110    & 7.9055   & 9.2990   & 0.9766    \\
        ncvx SVRG                       & 2.8340    & 3.4525   & 4.3140   & 0.3741    \\
        ncvx SVRG-SBB$_0$               & 0.4360    & 0.7962   & 4.3630   & 0.8396    \\
        ncvx SVRG-SBB$_\epsilon$-$1$    & 0.5100    & 0.5935   & 0.7660   & 0.0628    \\
        ncvx SVRG-SBB$_\epsilon$-$5$    & 0.2930    & 0.3541   & 0.4520   & 0.0362    \\
        ncvx SVRG-SBB$_\epsilon$-$10$   & \textbf{0.3070} & \textbf{0.3590} & \textbf{0.4740} & \textbf{0.0412} \\
        ncvx SVRG-SBB$_\epsilon$-$20$   & 0.3120    & 0.4104   & 0.5310   & 0.0454    \\
        ncvx SVRG-SBB$_\epsilon$-$50$   & 0.3319    & 0.4421   & 0.5641   & 0.0468   \\
        ncvx SVRG-SBB$_\epsilon$-$100$  & 0.4973    & 0.5802   & 0.7531   & 0.0593    \\
      \bottomrule
    \end{tabular}
    \label{tab:subtable_tste}
  \end{subtable}
  \caption{\footnotesize Computational complexity (second) comparisons on the synthetic dataset. `-' represents that the generalization error of the method can not be  lower than a predefined threshold in our setting, e.g. $\textit{error}_{gen} \leq 0.15$, before the algorithm calls a fixed number of IFO subroutine. As the \textit{TSTE} adopts the heavy-tail kernel, its objective function is always non-convex. We note the Gram matrix formulation of TSTE as ``kernel''. The variants of the mini-batch size verify the theoretical analysis. If the mini-batch size is large than some value, the computational efficiency of the proposed SVRG method will get worse as $b$ increasing.}
  \label{tab:table1}
\end{table*}

In this subsection, we use a small-scale synthetic dataset to analyze the performance of these methods in an idealized setting. Here, we provide sufficient ordinal information without noise to construct the embedding $\boldsymbol{X}$ in $\mathbb{R}^p$. One metrics are adopted to evaluate the generalization of different algorithms. Furthermore, the computational complexity is evaluated to illustrate the convergence behavior of each optimization method.
\\\\
\textbf{Settings. }The triplets of this dataset involve $n=100$ points in $10$-dimension Euclidean space as $\{{\boldsymbol{x}}_i\}_{i=1}^{100}\subset\mathbb{R}^{10}$. These data points $\{{\boldsymbol{x}}_i\}_{i=1}^{100}$ are independent and identically distributed (\textit{i.i.d}) random variables as ${\boldsymbol{x}}_i\sim\mathcal{N}(\boldsymbol{\mu}, \frac{1}{20}\boldsymbol{I})$, and $\boldsymbol{I}\in \mathbb{R}^{10\times 10}$ is the identity matrix. As the convex formulation needs the data points to satisfy the ``zero mean / centered assumption'': $\sum_{i=1}^{100}{\boldsymbol{x}}_i = \boldsymbol{0}$, we set $\boldsymbol{\mu}=\boldsymbol{0}$. The possible triple-wise similarity comparisons are generated based on the Euclidean distances between these samples. As it is known that the generalization error of the estimated Gram matrix $\boldsymbol{G}$ can be bounded if the triplets sample complexity is ${\bf O}(pn\log n)$ \cite{NIPS2016_6554}, we randomly choose $|\mathcal{Q}|=10,000$ triplets as the training set and the test set $\mathcal{Q}'$ has the same number of random sampled triplets. The regularization parameter and step size settings for the convex formulation follow the default setting of the \textit{STE / TSTE} implementation\footnote{\scriptsize\raggedright\url{http://homepage.tudelft.nl/19j49/ste/Stochastic_Triplet_Embedding.html}}. Note that we do not choose the step size by line search or the halving heuristic for convex formulation. The embedding dimension is selected just to be equal to $p=10$ without variations.
\\\\
\textbf{Evaluation Metrics. }The metrics that we evaluate various algorithms include the generalization error and running time. We split all triplets into a training set and a test set. Suppose $\mathcal{Y}_{{\mathcal{Q}}'}$ is the true label of test set ${\mathcal{Q}}'$ and the estimated label set is $\hat{\mathcal{Y}}_{{\mathcal{Q}}'}$. We adopt the learned embedding $\boldsymbol{X}$ from partial triple comparisons set $\mathcal{Q}\subset[n]^3$ to estimate the partial order of the unknown triplets. The percentage of held-out triplets whose labels $\hat{\mathcal{Y}}_{{\mathcal{Q}}'}$ are consistence with the true labels $\hat{\mathcal{Y}}_{{\mathcal{Q}}'}$ based on the embedding $\boldsymbol{X}$ is used as the main metric for evaluating the generalization of embedding $\boldsymbol{X}$, $\textit{error}_{gen} = \frac{1}{|{\mathcal{Q}}'|}\underset{{q}'\in{\mathcal{Q}}'}{\sum}\mathbbm{1}_{y_{{q}'}\ \neq\ \hat{y}_{{q}'}}$. The running time is the time spend to make the test error smaller than $0.15$
\\\\
\textbf{Competitors. }We evaluate both convex and non-convex formulations of four objective functions. We establish two baselines as :
\begin{itemize}
  \item the convex objective functions solved by projection gradient descent whose results are denoted as ``cvx'',
  \item the non-convex objection function solved by batch gradient descent whose results are denoted as ``ncvx''.
\end{itemize}

\begin{figure*}[thb!]
	\centering
	\begin{subfigure}{0.24\textwidth}
	{
		\includegraphics[width=\textwidth]{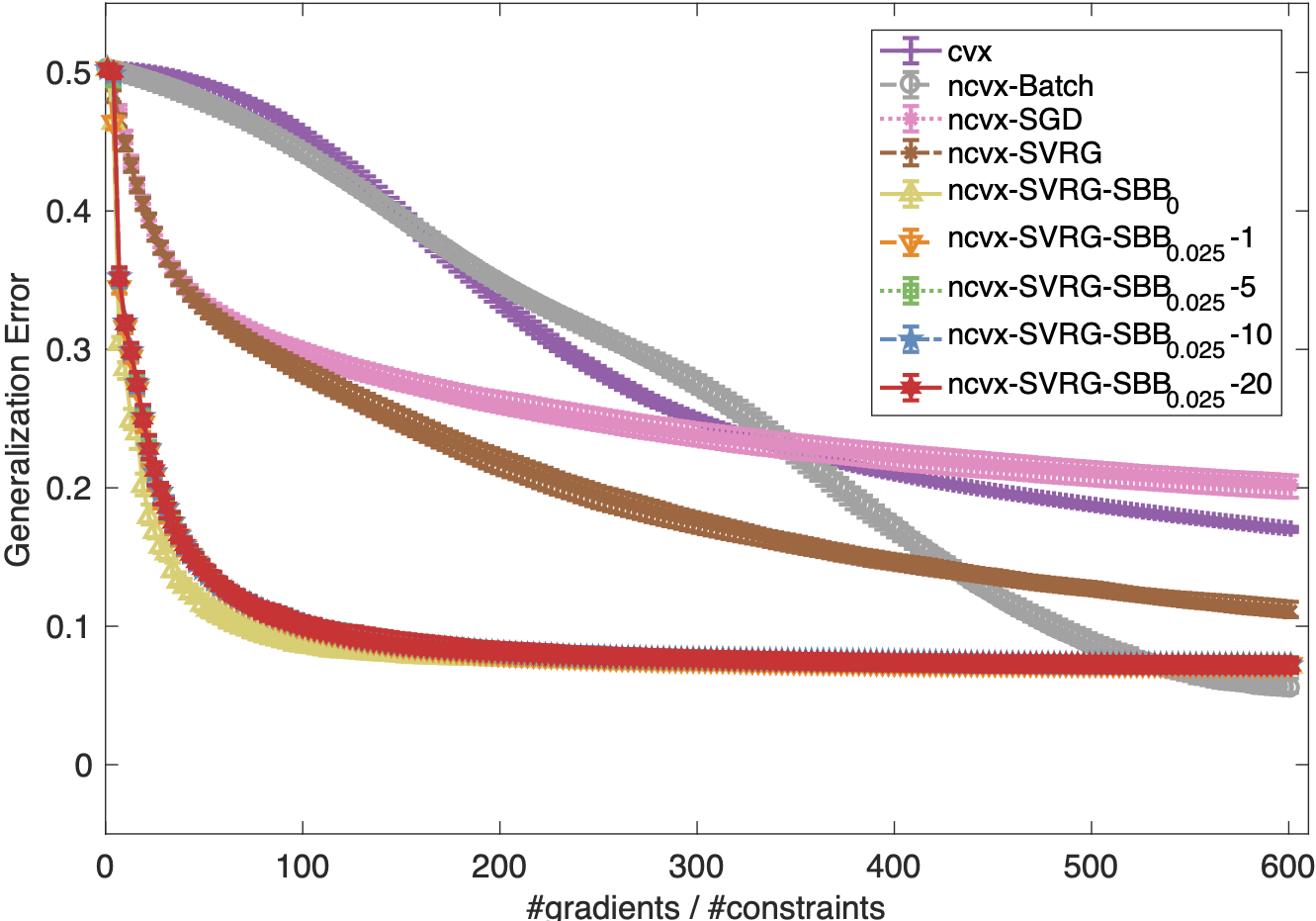}
		\caption{GNMDS}
		\label{fig:synthetic:gnmds}
	}
	\end{subfigure}
	\begin{subfigure}{0.24\textwidth}
	{
		\includegraphics[width=\textwidth]{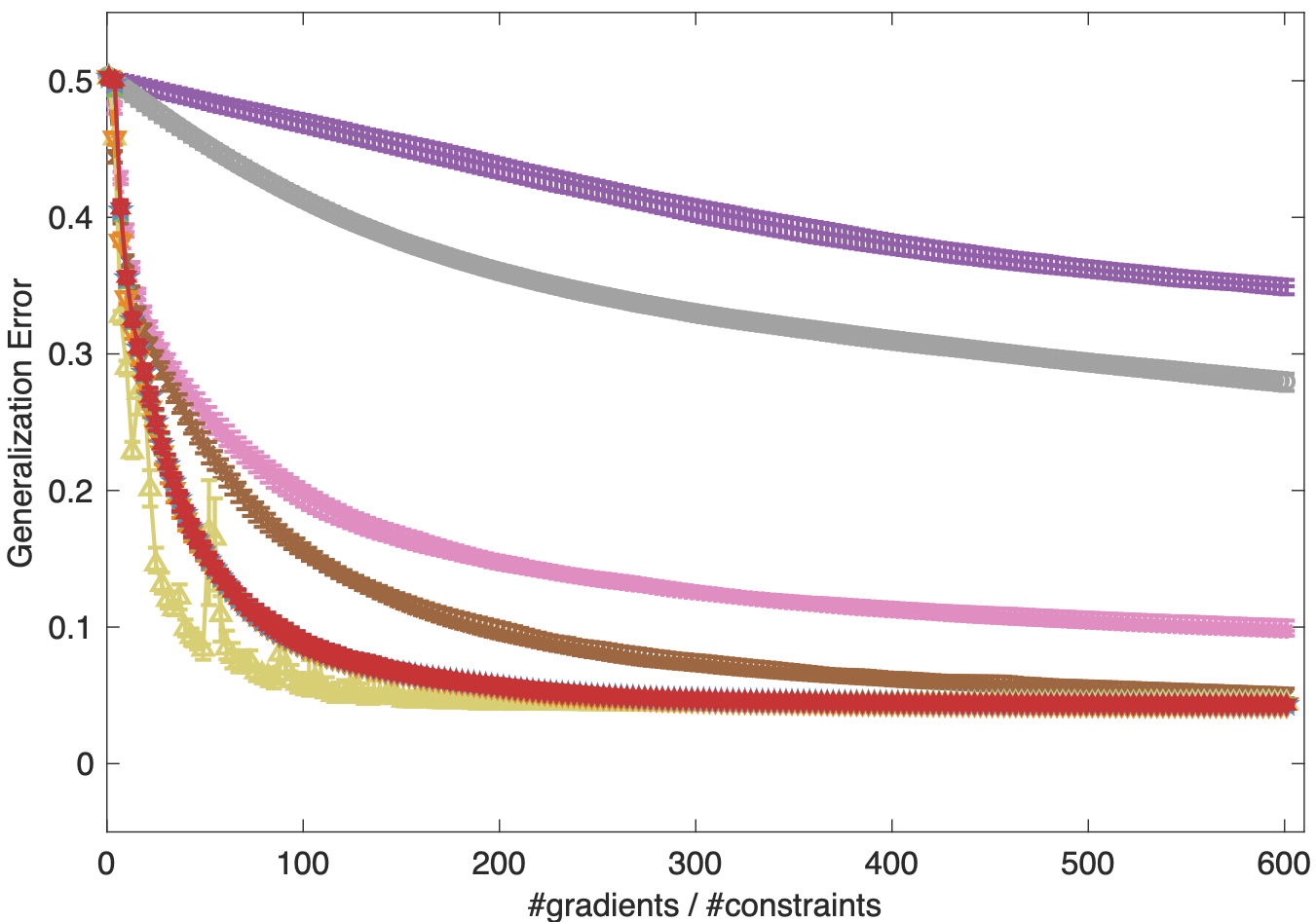}
		\caption{CKL}
		\label{fig:synthetic:ckl}
	}
	\end{subfigure}
	\begin{subfigure}{0.24\textwidth}
	{
		\includegraphics[width=\textwidth]{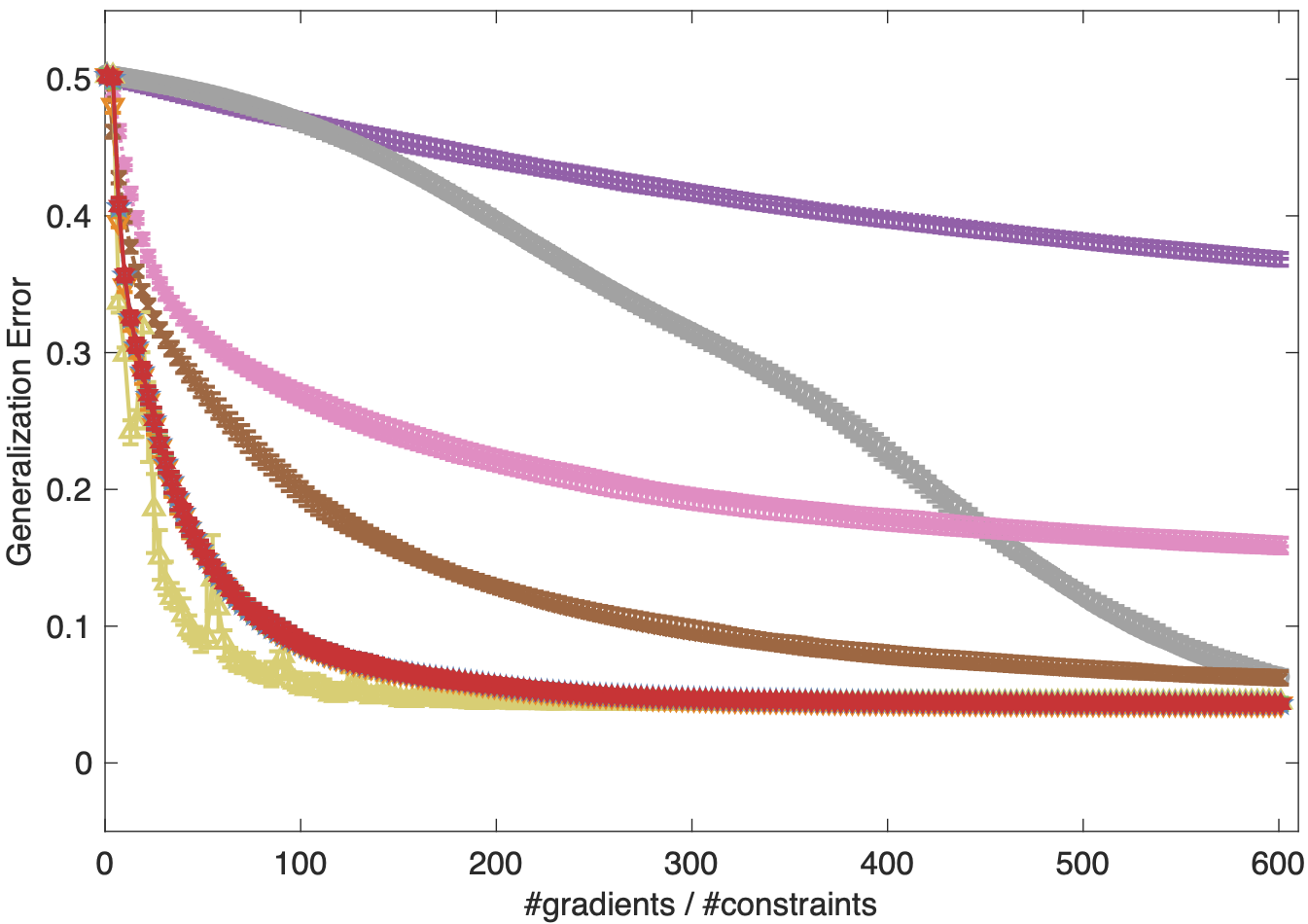}
		\caption{STE}
		\label{fig:synthetic:ste}
	}
	\end{subfigure}
	\begin{subfigure}{0.24\textwidth}
	{
		\includegraphics[width=\textwidth]{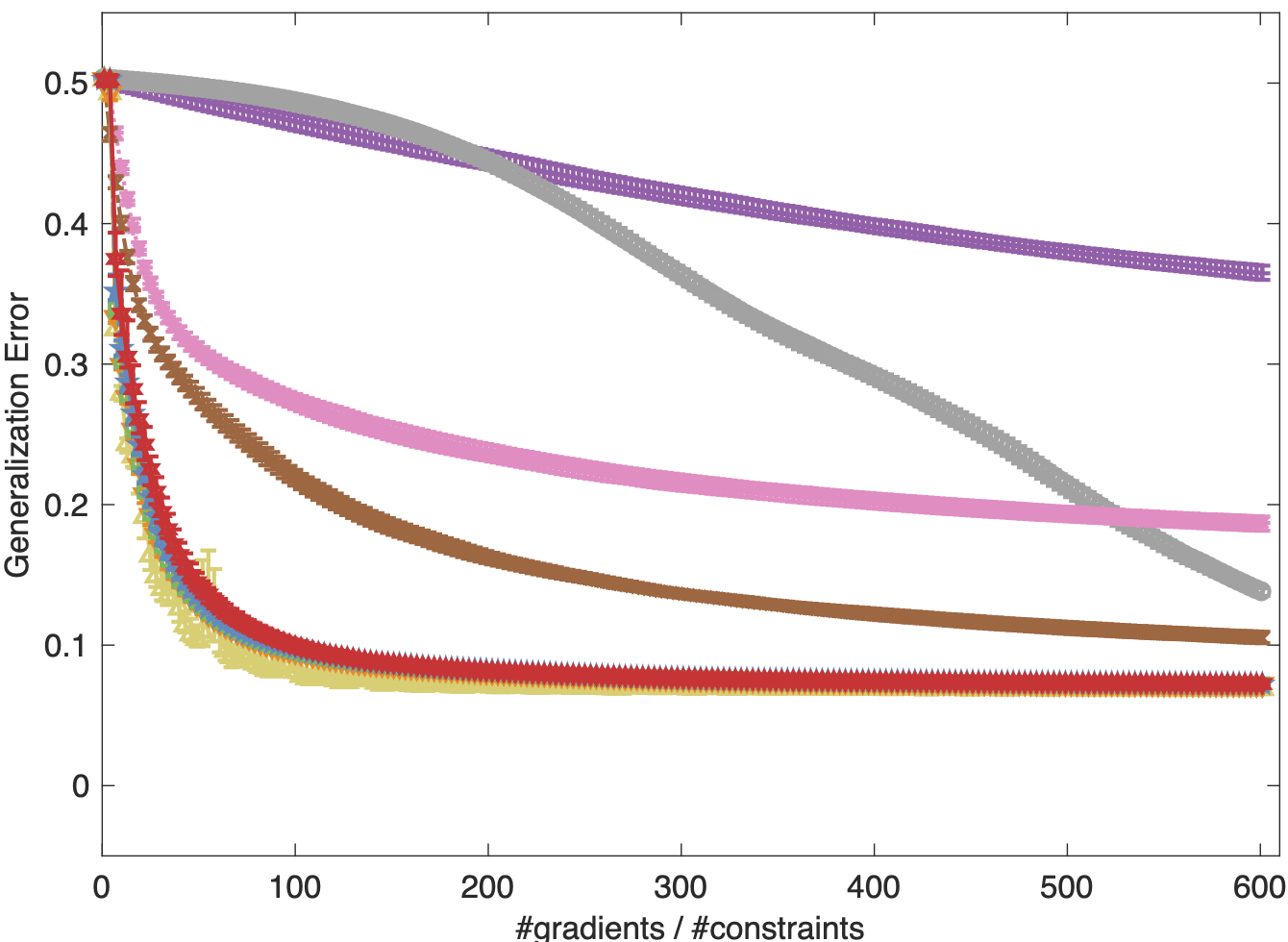}
		\caption{TSTE}
		\label{fig:synthetic:tste}
	}
	\end{subfigure}
	\caption{Generalization errors of SGD, SVRG, SVRG-SBB and batch methods on the synthetic dataset.}
	\label{fig:synthetic} 
\end{figure*}

We compare the performance of \textit{SVRG-SBB}$_\epsilon$ and its mini-batch variants with \textit{SGD}, fixed step size \textit{SVRG} (called \textit{SVRG} for short henceforth) as well as the two batch gradient descent methods. We compare these algorithms in the ``epoch'' sense which means that each method executes the gradient calculation by the same times in every epoch. To be concrete, as \textit{SVRG} and \textit{SVRG-SBB}$_\epsilon$ evaluate $2m+|\mathcal{Q}|$ times of (sub)gradient in each epoch, the batch and \textit{SGD} solutions query the same numbers of (sub)gradient. The mini-batch \textit{SVRG-SBB}$_\epsilon$ evaluates $2|\mathcal{I}_t|$ gradients in the inner loop where $|\mathcal{I}_t|$ is the size of mini-batch. We reduce the inner iteration of the mini-batch \textit{SVRG-SBB}$_\epsilon$ to $m/|\mathcal{I}_t|$ for fair comparisons.
\\\\
\textbf{Results. }In Figure~\ref{fig:synthetic}, the $x$-axis is the number of gradient calculation divided by the total number of training samples $|\mathcal{Q}|$. We set $m=|\mathcal{Q}|$ for \textit{SVRG}, \textit{SVRG-SBB}$_\epsilon$ and its min-batch variants. As a consequence, we evaluate the generalization error of each optimization method $3$ times in each epoch. The $y$-axis represents the generalization error. The results are based on $50$ trials with different initialization $\boldsymbol{X}_0$. The median of generalization error over $50$ trials with $[0.25, 0.75]$ confidence interval are plotted. The whole experiment results, generalization error and the computational complexity are shown in Figure \ref{fig:synthetic} and Table \ref{tab:table1}.

We observe the following phenomena from Figure \ref{fig:synthetic}. First of all, due to instability of SBB$_0$ step size, the SVRG-SBB$_0$ cause the generalization error to increase. The numerical `explosion' of SBB$_0$ step size leads to the failure of gradient decent method. This disadvantage of \textit{SBB}$_0$ is consistent with our theoretical results and insights. On the other hand, the results of the proposed \textit{SBB}$_\epsilon$ and its mini-batch variant do not vibrate during the whole process. This is the main motivation of the proposed stabilized method. Although the \textit{SBB}$_\epsilon$ method applies the conservative treatment and sacrifices some efficiency, this trade-off is valuable if the objective function does not hold the elegant properties, namely the condition number of Hessian matrix is not too large. Secondly, the \textit{SBB}$_\epsilon$ and its mini-batch variant outperform the SVRG incorporated with fixed step size by clear margins. The \textit{SBB}$_\epsilon$ method can choose more appropriate step size. Thus we can run the algorithm without adding much computational burden by tuning parameter. Just a fixed step size ensures the convergence of non-convex SVRG \cite{pmlr-v48-reddi16} but this particular step size is related to the Lipschitz constant of the objective function which is hard to obtain in practical problems, especially when the objective functions are non-convex. The \textit{SBB}$_\epsilon$ method is derived from BB step size and the latter one is obtained form the so-called ``secant equation'' which approximates the inverse of Hessian matrix by the identity matrix multiplied the desired step size. Such an approximation leads to the instability of \textit{BB} method in  stochastic non-convex optimization as the curvature condition \eqref{eq:curvature} does not always hold. A possible future direction is how to preserve the positive-definiteness of the inverse of the Hessian matrix without line search and extremely laborious computation in stochastic paradigm \cite{wang2017stochastic}. Moreover, all the stochastic methods including \textit{SGD}, \textit{SVRG} and \textit{SVRG-SBB}$_{\epsilon}$ ($\epsilon =0$ or $\epsilon>0$) converge fast at the initial several epochs and quickly get admissible results in terms of the relatively small generalization error.

Table \ref{tab:table1} shows the computational complexity achieved by batch and stochastic gradient descent with its variants, SGD, SVRR and SVRG-SBB$_\epsilon$ for convex and non-convex objective functions. All results are obtained with MATLAB$^\text{\textregistered{}}$ R2016b, on a desktop PC with Windows$^\text{\textregistered{}}$ $7$ SP$1$ $64$ bit, with $3.3$ GHz Intel$^\text{\textregistered{}}$ Xeon$^\text{\textregistered{}}$ E3-1226 v3 CPU, and $32$ GB $1600$ MHz DDR3 memory. It is clear to see that for all objective functions, SVRG-SBB$_\epsilon$ ($\epsilon =0$ or $\epsilon>0$) gains speed-up compared to the other methods. The superiority of SVRG-SBB$_\epsilon$ mini-batch in terms of the CPU time can also be observed from Table \ref{tab:table1}. Specifically, the speedup of SVRG-SBB$_\epsilon$ mini-batch over SVRG is about at least 4 times for all four models.

\subsection{Visualization of Eurodist Dataset}
The second empirical study is to visualize some objects in $2$d space based on their relative similarity comparisons.
\\\\
\textbf{Settings. }
The ``eurodist'' dataset describes the ``driving'' distances (in km) between $21$ cities of Europe, and is available in the stats library of \textbf{R}. In this dataset, there are $21,945$ possible quadruplets (i.e. $(i,j,l,k)$) in total. Here we adopt the ``local'' setting as only the triplets are utilized where $j=i$. A triplet $(i, j, k)$ represents a partial order as $d^2_{ij} < d^2_{ik}$, which indicates that ``the distance between cities $i$ and $j$ is smaller than the distance between cities $i$ and $k$'' and $d^2_{ij}$ is the road distance between cities $i$ and $j$ as $i,j\in\{1,...,21\}$. The main task of this dataset is to visualize the embedding of these $21$ cities in $2$-dimensional Euclidean space. We randomly sample $2,000$ triplets as the training set and the left are test set.
\begin{figure}[thb!]
	\centering
	\begin{subfigure}{0.20\textwidth}
	{
		\includegraphics[width=\textwidth]{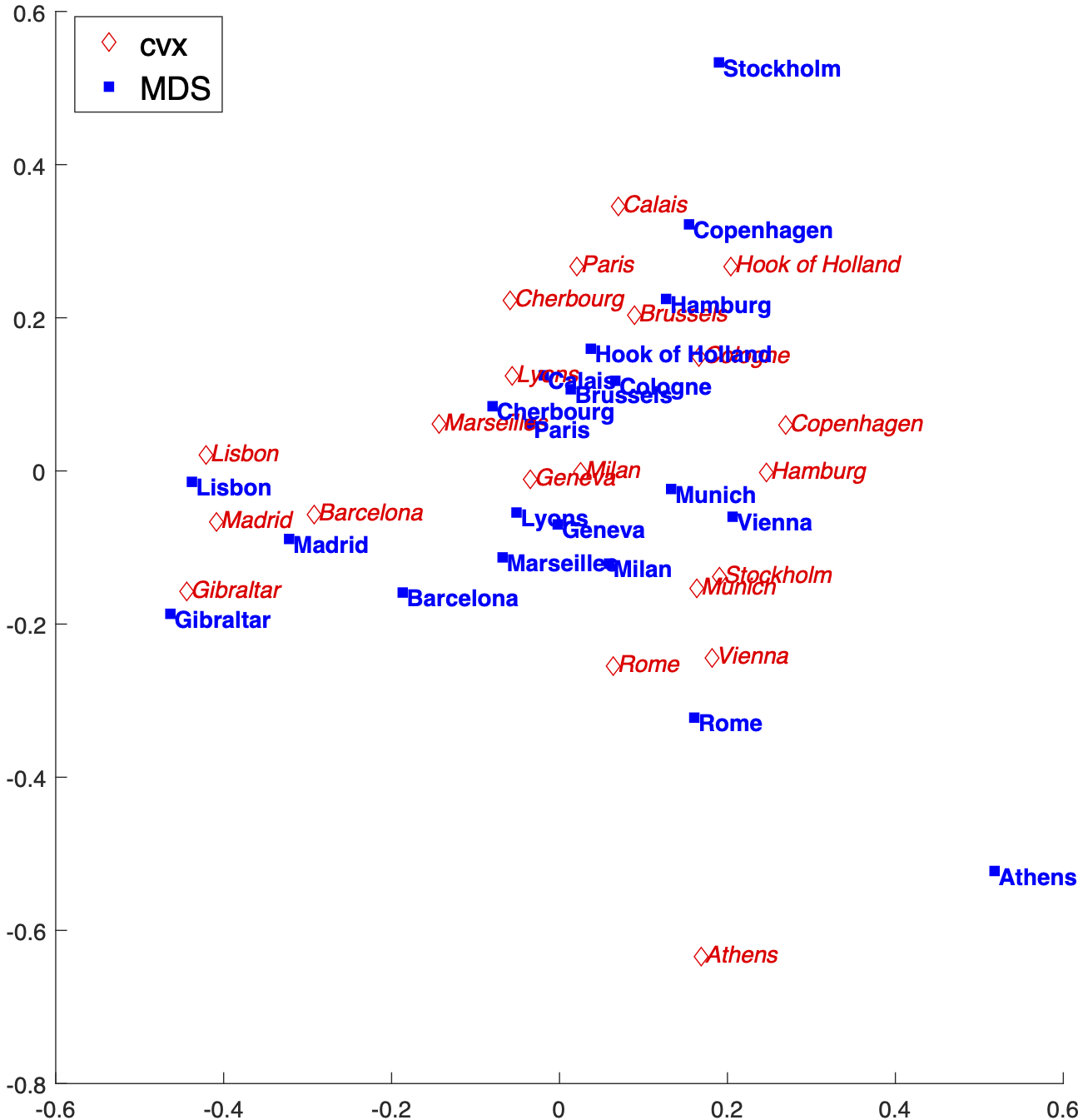}
		\caption{Convex Result}
		\label{fig:eurodist:cvx}
	}
	\end{subfigure}\vspace{0.5cm}
	\begin{subfigure}{0.20\textwidth}
	{
		\includegraphics[width=\textwidth]{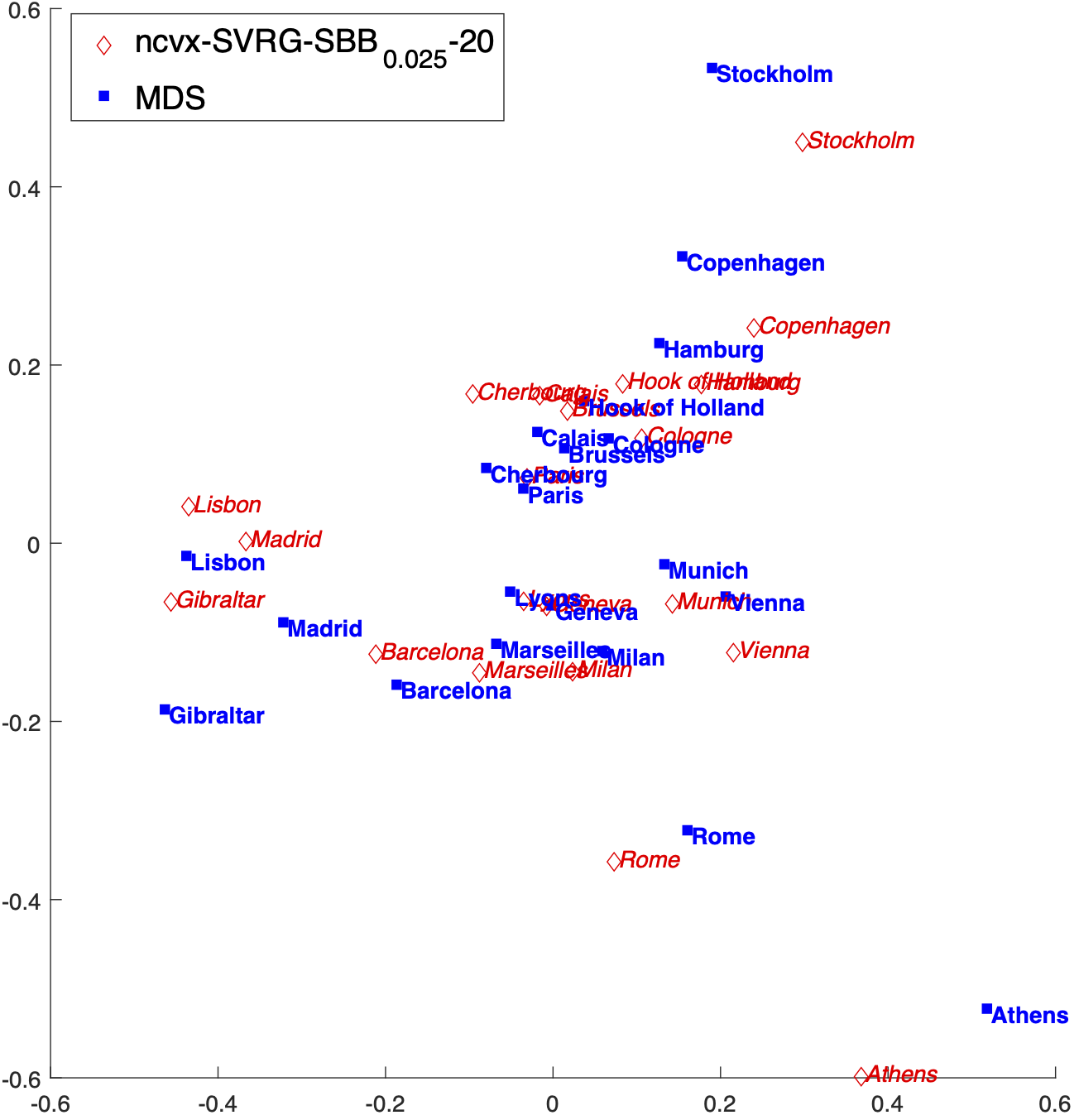}
		\caption{Non-Convex Result}
		\label{fig:eurodist:ncvx}
	}
	\end{subfigure}
	\caption{Visualization of the Eurodist dataset.}
	\label{fig:eurodist} 
\end{figure}
\\\\
\textbf{Competitors. }We compare the ordinal embedding results of all four models with the classical metric Multidimensional Scaling (\textit{MDS}) result. \textit{mMDS} is also known as Principal Coordinates Analysis (\textit{PCoA}) or Torgerson–Gower Scaling. It takes an input matrix giving dissimilarities between pairs of objects and outputs a coordinate matrix. Here we employ the road distance between a pair of two cities as their dissimilarities and obtain $2$d coordinates of each cities.  Note that there is no perfect embedding in the 2-dimensional space as the given distances are actually geodesic.
\\\\
\textbf{Results. }
Figure \ref{fig:eurodist} displays the Procrustes rotated embedding results of \textit{MDS} and ordinal embedding. Obviously, the full, explicit distance information helps \textit{MDS} to gain a better visualization. \textit{ODE} methods only adopt partial, relative comparisons and their visualizations are inferior to the competitor's result. However, the non-convex ordinal embedding methods still generate the reasonable representations of these $21$ cities. In the new coordinate system, all the positions are not contrary to geographical knowledge. Furthermore, the stochastic paradigm in \textit{SVRG-SBB}$_\epsilon$ dose not affect the quality of the embedding.

\subsection{Music Artists Similarity Comparison}
\begin{figure*}[thb!]
	\centering
	\begin{subfigure}{0.24\textwidth}
	{
		\includegraphics[width=\textwidth]{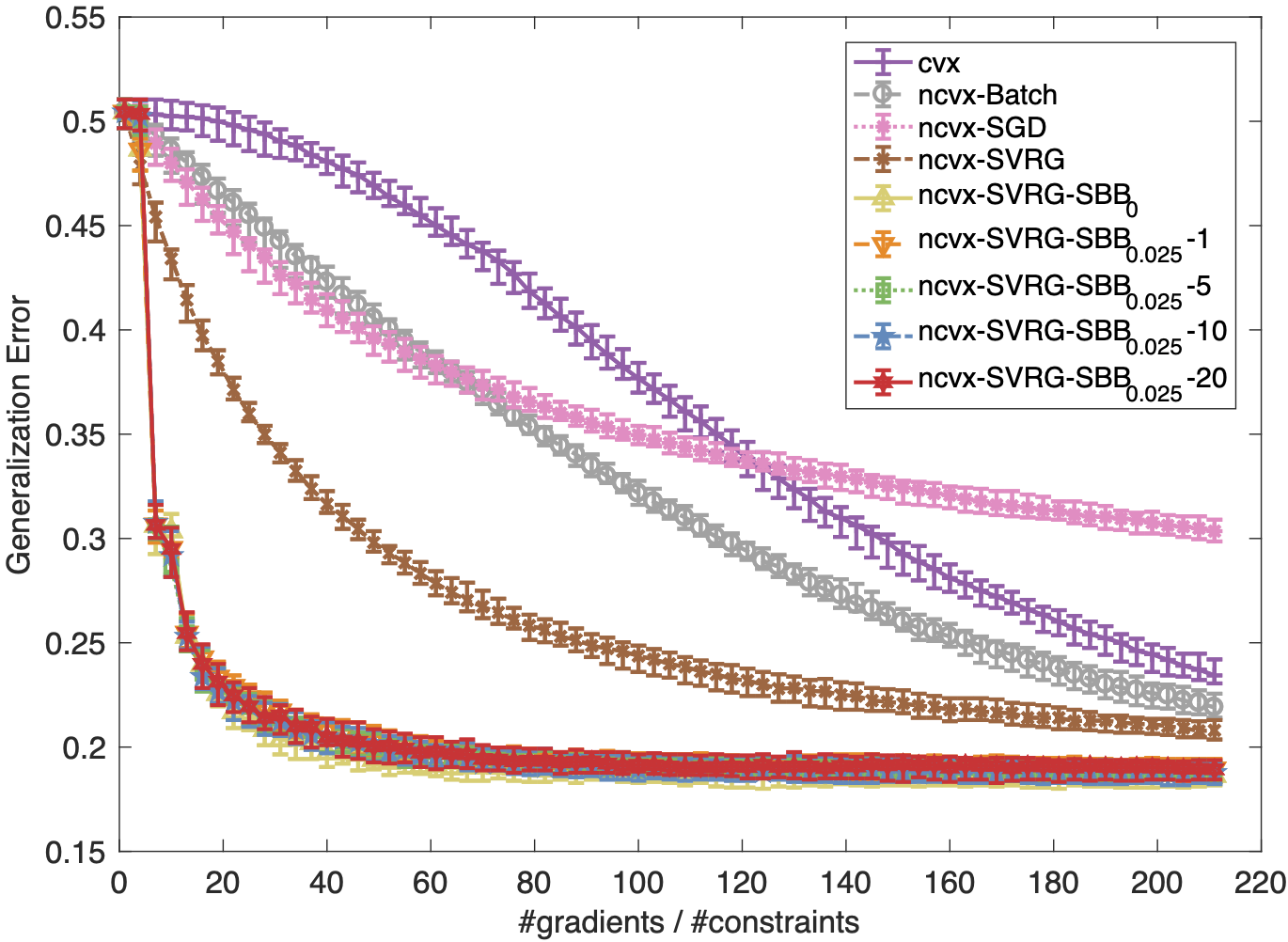}
		\caption{GNMDS}
		\label{fig:music:gnmds}
	}
	\end{subfigure}
	\begin{subfigure}{0.24\textwidth}
	{
		\includegraphics[width=\textwidth]{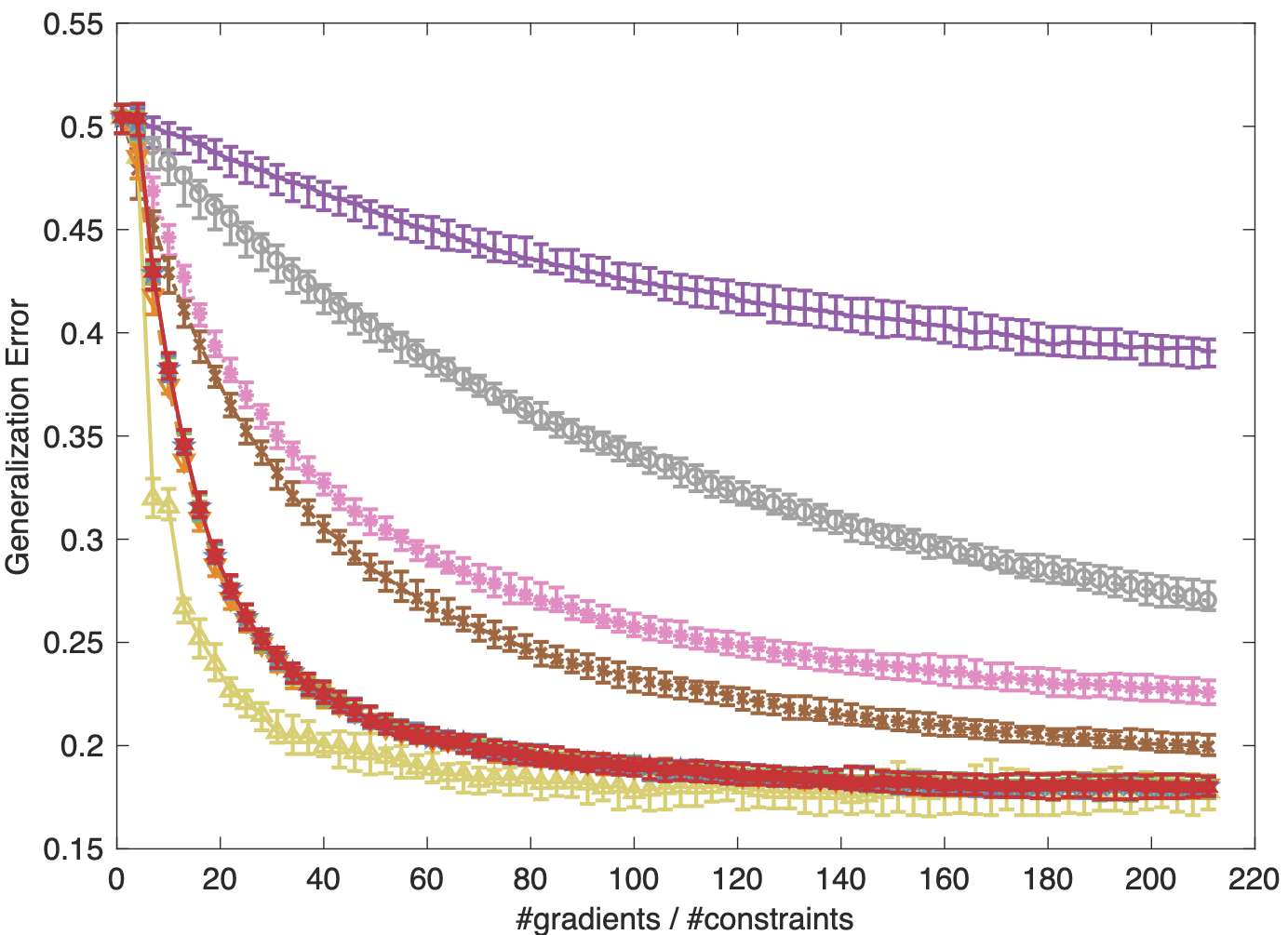}
		\caption{CKL}
		\label{fig:music:ckl}
	}
	\end{subfigure}
	\begin{subfigure}{0.24\textwidth}
	{
		\includegraphics[width=\textwidth]{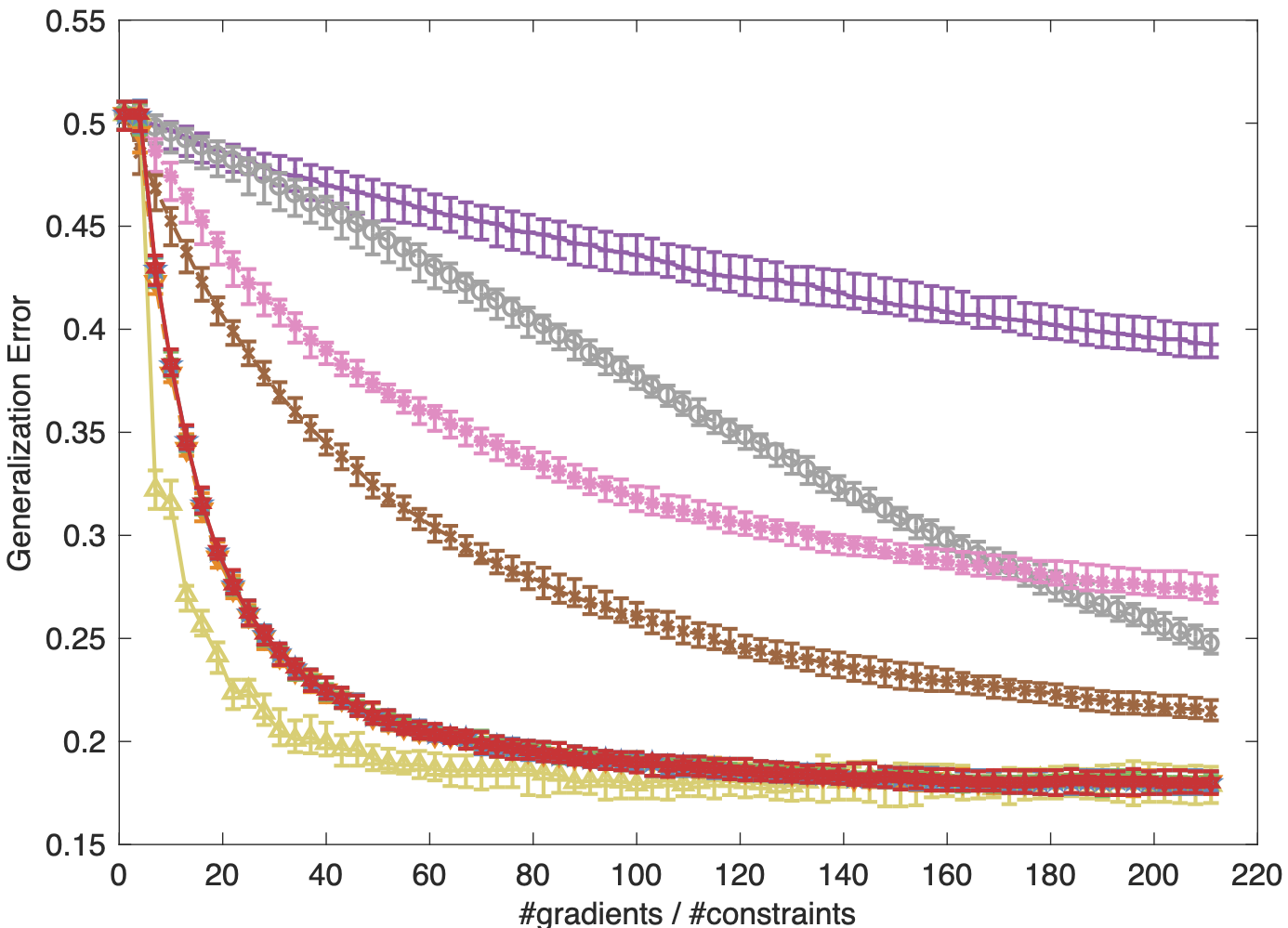}
		\caption{STE}
		\label{fig:music:ste}
	}
	\end{subfigure}
	\begin{subfigure}{0.24\textwidth}
	{
		\includegraphics[width=\textwidth]{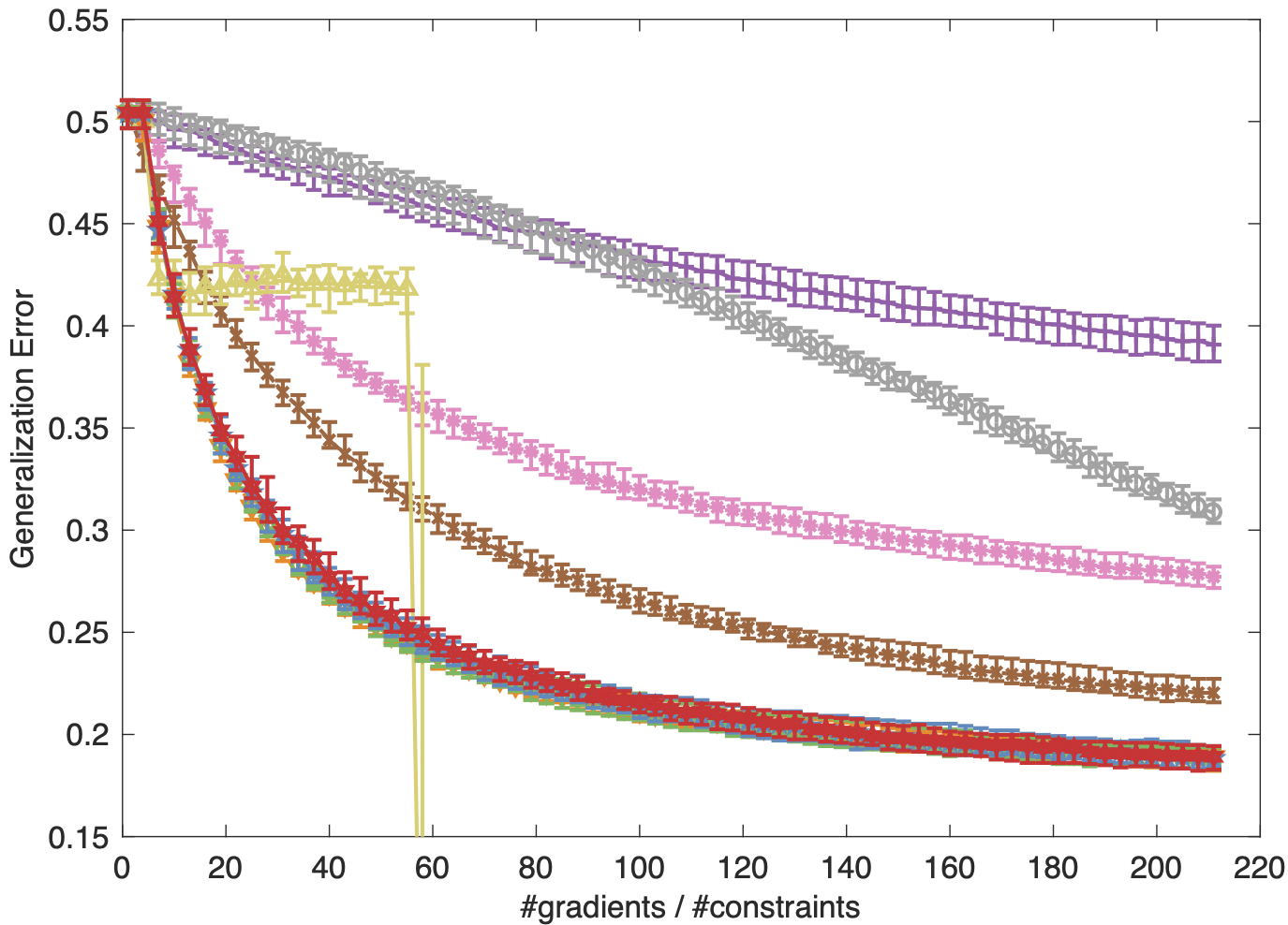}
		\caption{TSTE}
		\label{fig:music:tste}
	}
	\end{subfigure}
	\caption{Generalization errors of SGD, SVRG, SVRG-SBB and batch methods on the music artists dataset.}
	\label{fig:music} 
\end{figure*}

We implement all methods on a medium-scale real world dataset, called \textit{music artist similarity dataset}, which is collected by \cite{ellis2002quest} through a web-based survey.
\\\\
\textbf{Settings. }In this dataset, there are $412$ music artists involved in triple-wise comparisons based on music genre. The genre labels for all artists are gathered using Wikipedia\footnote{\url{https://www.wikipedia.org}}, to distinguish nine music genres (rock, metal, pop, dance, hip hop, jazz, country, gospel, and reggae). The similarity comparisons are labeled by $1,032$ participants. The number of triplets on the similarity between music artists is $213,472$. A triplet $(i, j, k)$ indicates that ``\textit{music artist $i$ is more similar to artist $j$ than artist $k$}''. Specifically, we use the data pre-processed by \cite{vandermaaten2012stochastic} via removing the inconsistent triplets from the original dataset. There are $9,107$ triplets for $N=400$ artists. We randomly sample $80$ percent of the comparisons as the training set and the left are the test data. The embedded dimension is $p=9$ as the number of genre category. All methods start with the same initialization $\boldsymbol{X}_0$ which is randomly generated.
\\\\
\textbf{Results. }
Each curve in Figure \ref{fig:music} shows the trend of test error of one method with respect to the epoch number. We execute 50 trials of each optimization method for the four objective functions. From Figure \ref{fig:music}, SBB$_\epsilon$ and its mini-batch variant can significantly speed up SVRG in terms of epoch number. Specially, the test error curves of four SVRG-SBB$_\epsilon$ ($\epsilon\geq 0$) methods decay much faster than those of SGD, SVRG and projection gradient descent at the initial epochs. We also observe that the SVRG-SBB$_0$ tends to failure as the step size is extremely large ($\sim10^{35}$). The SVRG-SBB$_\epsilon$ step size can  effectively avoid the occurrence of similar situations.

\subsection{Image Retrieval on SUN397}
\begin{table*}
  \centering
  \begin{subtable}[t]{0.45\textwidth}
    \caption{GNMDS}
    \begin{tabular}{l|cccc}
    \toprule
    \multicolumn{1}{c|}{\multirow{2}[4]{*}{method}} & \multicolumn{2}{c}{5\%} & \multicolumn{2}{c}{10\%} \\
    \cmidrule{2-5}                                  & MAP         & P         & MAP        & P           \\
    \midrule
    cvx                                             & 0.0259      & 0.0712    & 0.0255     & 0.0694      \\
    ncvx Batch                                      & 0.0474      & 0.1326    & 0.0386     & 0.1119      \\
    ncvx SGD                                        & 0.3120      & 0.4606    & 0.1865     & 0.3359      \\
    ncvx SVRG                                       & 0.3460      & 0.4949    & 0.2112     & 0.3631      \\
    ncvx SVRG-SBB$_0$                               & 0.4659      & 0.5783    & 0.2971     & 0.4408      \\
    ncvx SVRG-SBB$_\epsilon$-$1$                    & 0.4861      & 0.5993    & \textbf{0.3085}     & 0.4533      \\
    ncvx SVRG-SBB$_\epsilon$-$5$                    & 0.4861      & 0.5995    & \textbf{0.3085}     & \textbf{0.4536}      \\
    ncvx SVRG-SBB$_\epsilon$-$10$                   & 0.4867      & 0.5998    & 0.3083     & 0.4534      \\
    ncvx SVRG-SBB$_\epsilon$-$20$                   & \textbf{0.4872}      & \textbf{0.6005}    & \textbf{0.3085}     & 0.4535      \\
    \bottomrule
    \end{tabular}
    \label{tab:subtable_gnmds_sun}
  \end{subtable}
  \vspace{\fill}
  \begin{subtable}[t]{0.45\textwidth}
    \caption{CKL}
    \begin{tabular}{l|cccc}
    \toprule
    \multicolumn{1}{c|}{\multirow{2}[4]{*}{method}} & \multicolumn{2}{c}{5\%} & \multicolumn{2}{c}{10\%} \\
    \cmidrule{2-5}                                  & MAP         & P         & MAP        & P           \\
    \midrule
    cvx                                             & 0.0258      & 0.0704    & 0.0260     & 0.0711      \\
    ncvx Batch                                      & 0.0376      & 0.1087    & 0.0338     & 0.0992      \\
    ncvx SGD                                        & 0.4830      & 0.5926    & 0.1765     & 0.3192      \\
    ncvx SVRG                                       & 0.5445      & 0.6450    & 0.2149     & 0.3623      \\
    ncvx SVRG-SBB$_0$                               & 0.4662      & 0.5793    & 0.2037     & 0.3547      \\
    ncvx SVRG-SBB$_\epsilon$-$1$                    & 0.5645      & 0.6525    & 0.2805     & 0.4270      \\
    ncvx SVRG-SBB$_\epsilon$-$5$                    & 0.5652      & 0.6532    & 0.2809     & 0.4273      \\
    ncvx SVRG-SBB$_\epsilon$-$10$                   & \textbf{0.5653}      & \textbf{0.6533}    & \textbf{0.2810}     & \textbf{0.4274}      \\
    ncvx SVRG-SBB$_\epsilon$-$20$                   & \textbf{0.5653}      & \textbf{0.6533}    & \textbf{0.2810}     & \textbf{0.4274}      \\
    \bottomrule
    \end{tabular}
    \label{tab:subtable_ckl_sun}
  \end{subtable}

  \hspace{\fill}

  \begin{subtable}[t]{0.45\textwidth}
    \caption{STE}
    \begin{tabular}{l|cccc}
    \toprule
    \multicolumn{1}{c|}{\multirow{2}[4]{*}{method}} & \multicolumn{2}{c}{5\%} & \multicolumn{2}{c}{10\%} \\
    \cmidrule{2-5}                                  & MAP         & P         & MAP        & P           \\
    \midrule
    cvx                                             & 0.0257      & 0.0707    & 0.0262     & 0.0722      \\
    ncvx Batch                                      & 0.0304      & 0.0881    & 0.0306     & 0.0874      \\
    ncvx SGD                                        & 0.2820      & 0.4261    & 0.1948     & 0.3391      \\
    ncvx SVRG                                       & 0.5265      & 0.6302    & 0.3783     & 0.5105      \\
    ncvx SVRG-SBB$_0$                               & 0.4637      & 0.5799    & 0.0070     & 0.0555      \\
    ncvx SVRG-SBB$_\epsilon$-$1$                    & \textbf{0.6387}      & \textbf{0.7147}    & \textbf{0.4584}     & \textbf{0.5730}      \\
    ncvx SVRG-SBB$_\epsilon$-$5$                    & 0.6362      & 0.7129    & 0.4568     & 0.5719      \\
    ncvx SVRG-SBB$_\epsilon$-$10$                   & 0.6359      & 0.7127    & 0.4566     & 0.5719      \\
    ncvx SVRG-SBB$_\epsilon$-$20$                   & 0.6358      & 0.7126    & 0.4565     & 0.5717      \\
    \bottomrule
    \end{tabular}
    \label{tab:subtable_ste_sun}
  \end{subtable}
  \vspace{\fill}
  \begin{subtable}[t]{0.45\textwidth}
    \caption{TSTE}
    \begin{tabular}{l|cccc}
    \toprule
    \multicolumn{1}{c|}{\multirow{2}[4]{*}{method}} & \multicolumn{2}{c}{5\%} & \multicolumn{2}{c}{10\%} \\
    \cmidrule{2-5}                                  & MAP         & P         & MAP        & P           \\
    \midrule
    cvx                                             & 0.0257      & 0.0695    & 0.0261     & 0.0704      \\
    ncvx Batch                                      & 0.0270      & 0.0736    & 0.0273     & 0.0751      \\
    ncvx SGD                                        & 0.3864      & 0.5074    & 0.2381     & 0.3742      \\
    ncvx SVRG                                       & 0.7198      & 0.7746    & 0.5241     & 0.6221      \\
    ncvx SVRG-SBB$_0$                               & 0.0070      & 0.0555    & 0.0070     & 0.0555      \\
    ncvx SVRG-SBB$_\epsilon$-$1$                    & 0.8861      & \textbf{0.9034}    & 0.6859     & 0.7431      \\
    ncvx SVRG-SBB$_\epsilon$-$5$                    & \textbf{0.8898}      & 0.9030    & \textbf{0.6865}     & \textbf{0.7437}      \\
    ncvx SVRG-SBB$_\epsilon$-$10$                   & 0.8866      & 0.9000    & 0.6854     & 0.7432      \\
    ncvx SVRG-SBB$_\epsilon$-$20$                   & 0.8846      & 0.8995    & 0.6847     & 0.7426      \\
    \bottomrule
    \end{tabular}
    \label{tab:subtable_tste_sun}
  \end{subtable}
  \caption{Image retrieval performance (MAP and Precision@60) on SUN397}
  \label{tab:table2}
\end{table*}
We apply the ordinal embedding method with the proposed \textit{SVRG-SBB} algorithm on a real-world dataset, i.e., SUN 397. In the visual search task, we wish to see how the learned representation or embedding characterizes the ``relevance'' of the same image category and the ``discrimination'' of different image categories. Hence, we use the image representation obtained by ordinal embedding for image retrieval.
\\\\
\textbf{Settings. }We evaluate the effectiveness of the ordinal embedding methods for image retrieval on the SUN$397$ dataset. SUN$397$ consists of about $108,000$ images from $397$ scene categories. In SUN$397$, each image has a $1,600$-dimensional feature vector extracted by principle component analysis (\textit{PCA}) from $12,288$-dimensional Deep Convolution Activation Features \cite{Gong2014}. We form the training set by randomly sampling $1,080$ images from $18$ categories with $60$ images in each category. Only the training set is used for learning the representations from ordinal constraints and a nonlinear mapping from the original feature space to the embedded space. We denote the mapping as $\mathcal{M}:\mathbb{R}^{1600}\rightarrow\mathbb{R}^{18}$. The nonlinear mapping $\mathcal{M}$ is used to predict the embedded images in $\mathbb{R}^{18}$, which do not participate in the relative similarity comparisons. We use Regularized Least Square (\textit{RLS}) and \textit{RBF} kernel to solve the nonlinear mapping $\mathcal{M}$. The test set consists of $720$ images randomly chosen from $18$ categories with $40$ images in each category. We use labels of training images to generate the similarity comparisons. The ordinal constraints are generated like \cite{7410580}: we randomly sample two images $i,\ j$ which are from the same category and choose image $k$ from the left categories. As the semantic similarity between $i$ and $j$ in the same class is larger than the similarity between $i$ and $k$ in the different class, a triplet $(i, j, k)$ describes the relative similarity comparison. The total number of such triplets is $70,000$. Wrong triplets are then synthesized to simulate the human error in real-world data. We randomly sample $5\%, 10\%$ and $15\%$ triplets to exchange the positions of $j$ and $k$ in each triplet $(i,j,k)$.
\begin{figure*}[thb!]
  \centering
  \begin{subfigure}{0.20\textwidth}
  {
    \includegraphics[width=\textwidth]{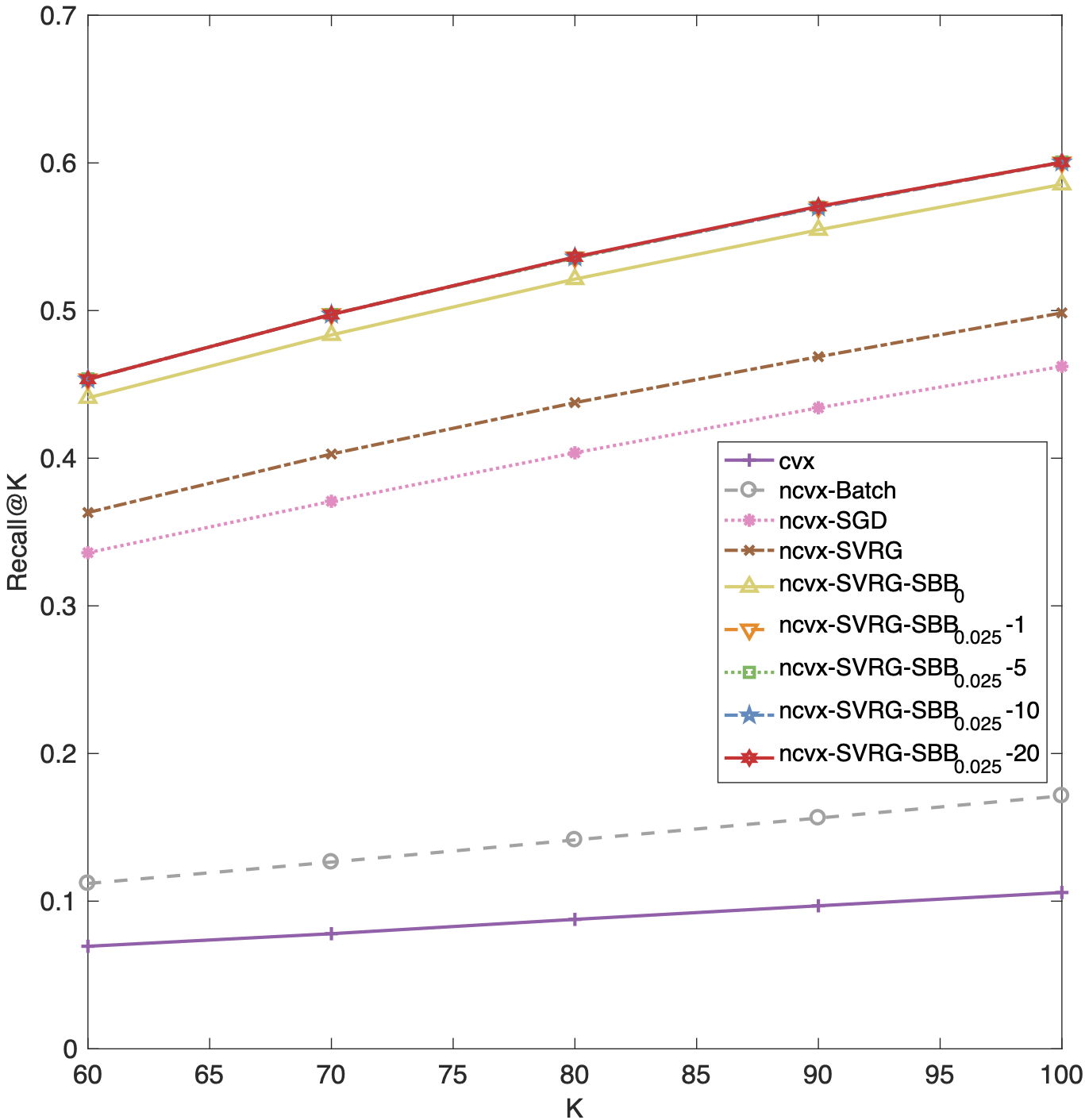}
    \caption{GNMDS}
    \label{fig:sun:gnmds}
  }
  \end{subfigure}
  \begin{subfigure}{0.20\textwidth}
  {
    \includegraphics[width=\textwidth]{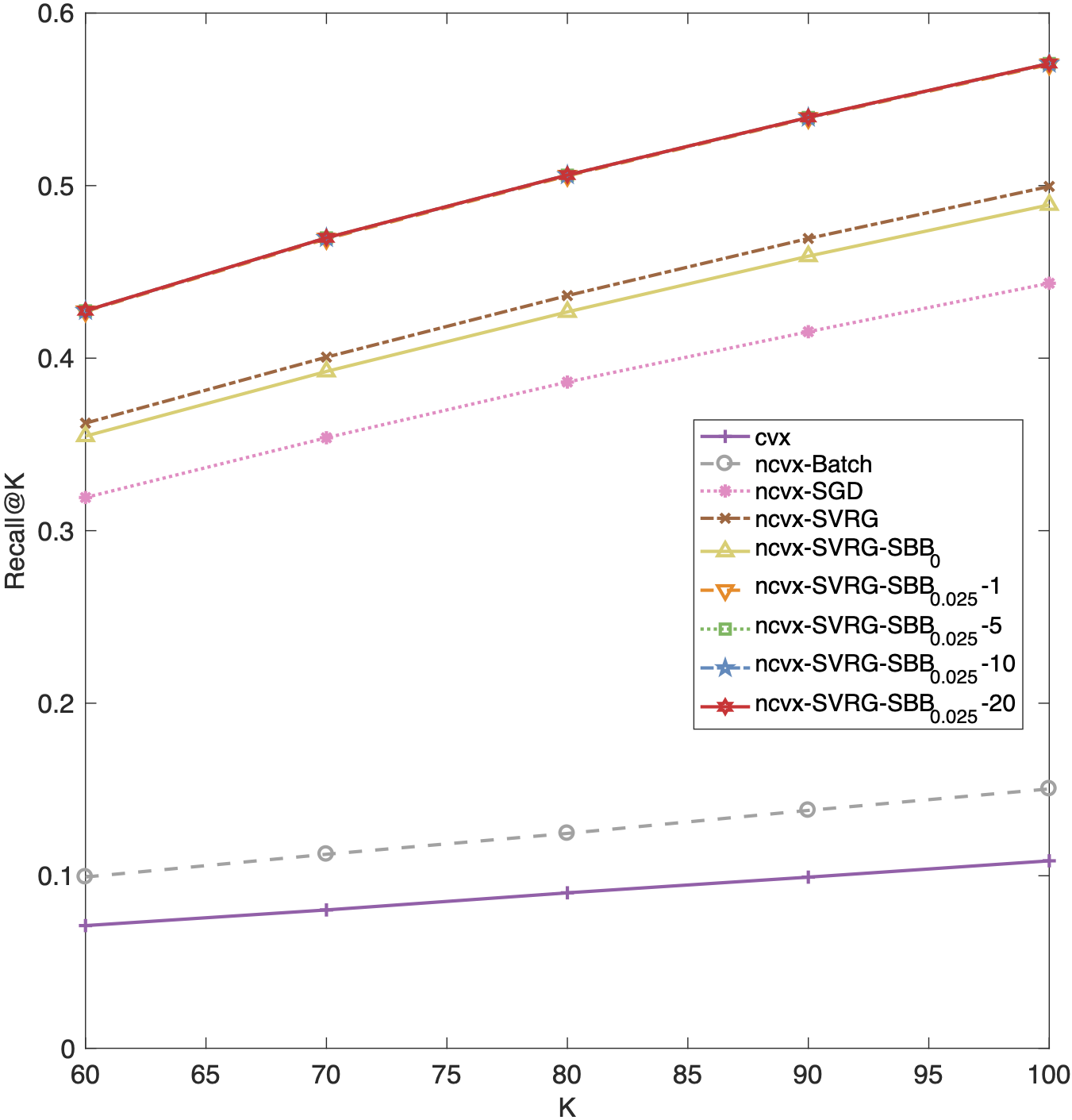}
    \caption{CKL}
    \label{fig:sun:ckl}
  }
  \end{subfigure}
  \begin{subfigure}{0.20\textwidth}
  {
    \includegraphics[width=\textwidth]{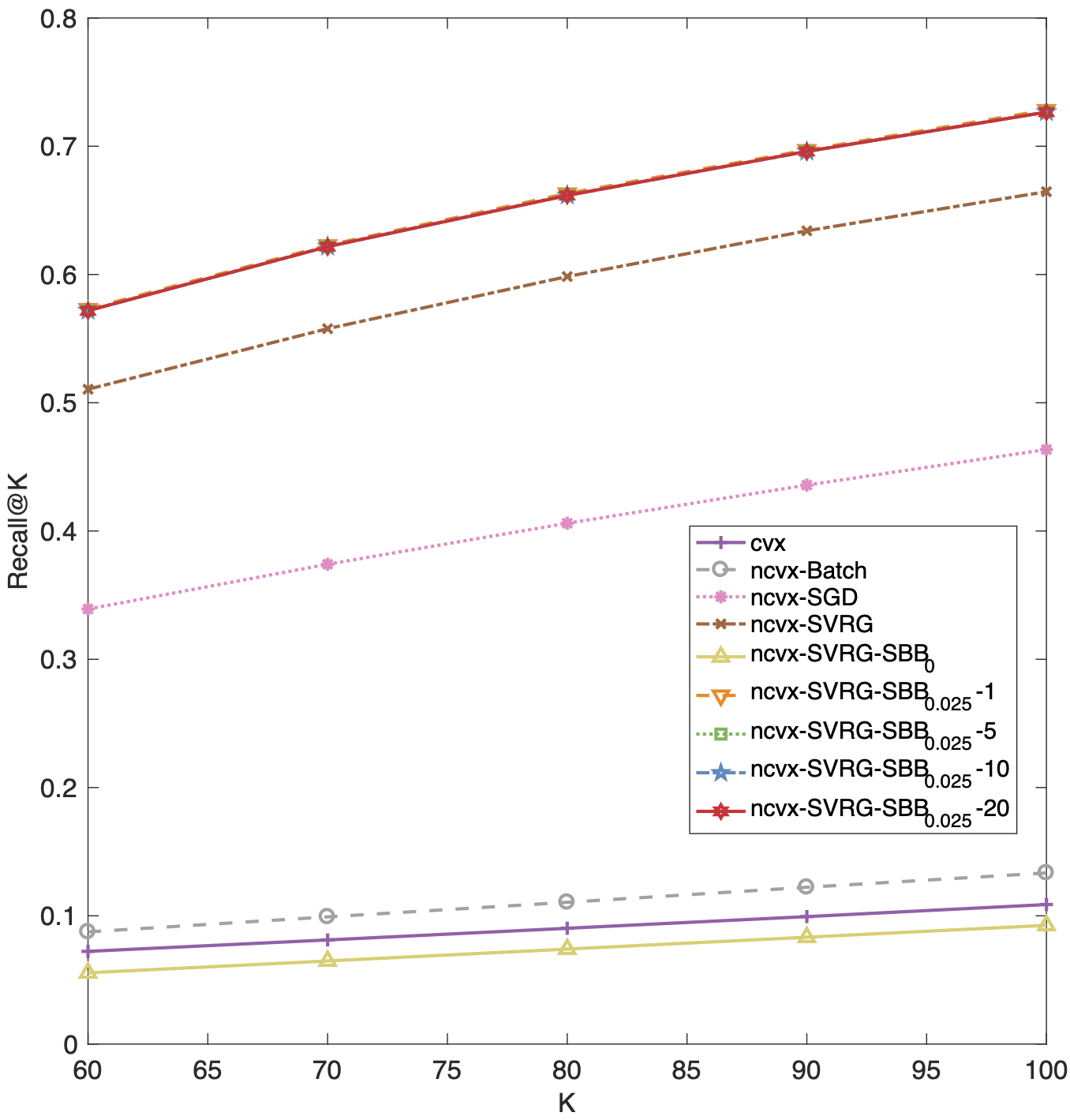}
    \caption{STE}
    \label{fig:sun:ste}
  }
  \end{subfigure}
  \begin{subfigure}{0.20\textwidth}
  {
    \includegraphics[width=\textwidth]{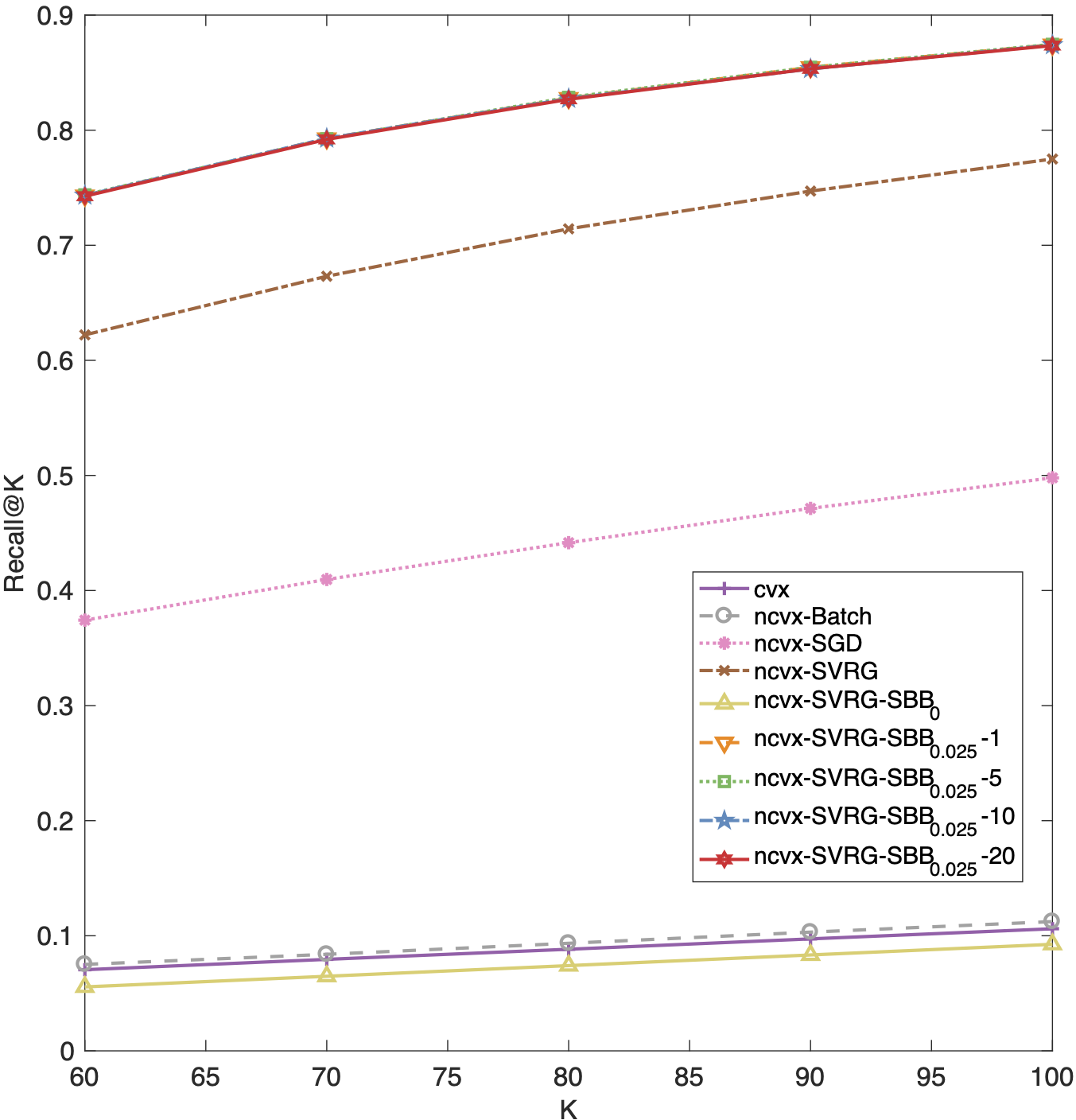}
    \caption{TSTE}
    \label{fig:sun:tste}
  }
  \end{subfigure}
  \caption{Recall@K with $10\%$ noise on SUN397.}
  \label{fig:sun} 
\end{figure*}
\\\\
\textbf{Evaluation Metrics. }To measure the effectiveness of various ordinal embedding methods for visual search, we consider three evaluation metrics, \textit{i.e.}, precision at top-K positions (\textit{Precision@K}), recall at top-K positions (\textit{Recall@K}), and Mean Average Precision (\textit{MAP}). Given the mapped feature $\boldsymbol{X}=\{\boldsymbol{x}_1, \boldsymbol{x}_2,\dots,\boldsymbol{x}_{720}\}\subset\mathbb{R}^{18}$ of test images and an image $i\in[n]$ belonging to the class $c_i$ as a query, we sort the images of training set according to the distances between their embedded feature in $\mathbb{R}^{18}$ and $\boldsymbol{x}_i$ in an ascending order as $\mathcal{R}_i$. True positives ($\textit{TP}^K_i$) are images correctly labeled as positives, which involve the images belonging to $c_i$ and listed within the top K positions in $\mathcal{R}_i$. False positives ($\text{FP}^K_i$) refer to negative examples incorrectly labeled as positives, which are the images belonging to $c_l(l\neq i)$ and listed within the top K in $\mathcal{R}_i$. True negatives ($\text{TN}^K_i$) correspond to negatives correctly labeled as negatives, which refer to the images belonging to $c_l(l\neq i)$ and listed after the top K in $\mathcal{R}_i$. Finally, false negatives ($\text{FN}^K_i$) refer to positive examples incorrectly labeled as negatives, which are relevant to the images belonging to class $c_i$ and listed after the top K in $\mathcal{R}_i$. We are able to define \textit{Precision@K} and \textit{Recall@K} as:
$
    \text{Precision@}K = \frac{1}{n}\sum_i^{n}p_i^K=\frac{1}{n}\sum_i^{n}\frac{\text{TP}^K_i}{\text{TP}^K_i+\text{FP}^K_i}
$
and
$
    \text{Recall@}K = \frac{1}{n}\sum_i^{n}r_i^K=\frac{1}{n}\sum_i^{n}\frac{\text{TP}^K_i}{\text{TP}^K_i+\text{FN}^K_i}.
$
These two measurements are both single-valued metric based on the top K ranking order of training images refered to the query image. It is also desirable to consider the total order of the training images which are in the same category as the query image. By computing precision and recall at every position in the ranked sequence for query $\boldsymbol{x}_i$, one can plot a precision-recall curve, plotting precision $p_i(r)$ as a function of recall $r_i$. Average Precision (\textit{AP}) computes the average value of $p_i(r)$ over the interval from $r_i=0$ to $r_i=1$:
$
	\text{AP}_i = \int_{0}^1 p_i(r_i)dr_i,
$
which is the area under precision-recall curve. This integral can be replaced with a finite sum over every position $s$ in the ranked sequence of the embedding:
$
\text{AP}_i = \sum_{s=1}^{40} p_i(s)\cdot\triangle r_i(s),
$
where $\triangle r_i(s)$ is the change in recall from items $s-1$ to $s$. The MAP used in this paper is defined as $\text{MAP} = \frac{1}{n}\sum_{i=1}^{n}\text{AP}_i$.
\\\\
\textbf{Results.} The experiment results are shown in Table \ref{tab:table2} and Figure \ref{fig:sun}. With $K$ varying from $40$ to $100$, we observe that non-convex \textit{SVRG-SBB}$_\epsilon$ consistently achieves the superior Precision@K, Recall@K and MAP results against the other methods with the same gradient calculation. The results illustrate that \textit{SVRG-SBB}$_\epsilon$ is more suitable for non-convex objective functions than \textit{SVRG-SBB}$_0$. \textit{SVRG-SBB}$_\epsilon$ has a very promising potential in practice, because it generates appropriate step sizes automatically while running the algorithm and the result is robust. Moreover, under our setting, all the ordinal embedding methods achieve reasonable results for image retrieval. It illustrates that high-quality relative similarity comparisons can be used for learning meaningful representation of massive data, thereby making it easier to extract useful information in other applications.

\section{Conclusions}

In this paper, we propose a stochastic non-convex framework for the ordinal embedding problem. A novel stochastic gradient descent algorithm called \textit{SVRG-SBB} is applied to solve this non-convex formulation. The proposed \textit{SVRG-SBB} is a variant of \textit{SVRG} method. It incorporate with the so-called stabilized \textit{BB} (\textit{SBB}) step size, a new, stable and adaptive step size introduced in this paper. The motivation of the \textit{SBB} step size is to overcome the instability of the original \textit{BB} step size when the strongly convexity is absent. We also establish the $O(1/T)$ convergence rate of \textit{SVRG-SBB}. Such a convergence rate is comparable to the existing best convergence results of \textit{SVRG} in the literature. Furthermore, we derive the analysis to mini-batch variants of \textit{SVRG-SBB}. We also analyze the \textit{PL} function on which \textit{SVRG-SBB} attains linear convergence to the global optimum. A series of simulations and real-world data experiments are implemented to demonstrate the effectiveness of the proposed \textit{SVRG-SBB} for the ordinal embedding problem. The proposed \textit{SVRG-SBB} outperforms most of the state-of-the-art methods from the perspective of computational cost.



{
  \small
  \bibliographystyle{IEEEtran}
  \bibliography{IEEEtran}
}

{
\onecolumn
\appendix
\section*{A. Proof of Lemma \ref{keylemma-minibatch}}
\label{appendix:Lemma 1}
To prove Lemma \ref{keylemma-minibatch}, we need the following lemma,
which establishes the bound of the norm of variance gradient ${\boldsymbol{u}}_t^{s+1}$.
\begin{lemma}[Bounded $\mathbb{E}\left(\|{\boldsymbol{u}}_t^{s+1}\|^2\right)$]
\label{lemm:bound-ut}
Under assumptions of Lemma \ref{keylemma-minibatch}, the following holds
\begin{align*}
&\mathbb{E}\left[\left\|\boldsymbol{x}_t^{s+1}\right\|^2\right]
\leq 2\mathbb{E}\left[\left\|\nabla f(\boldsymbol{x}_t^{s+1})\right\|^2\right] + \frac{2L^2}{b} \mathbb{E}\left[\left\|\boldsymbol{x}_t^{s+1} - \tilde{\boldsymbol{x}}^s\right\|^2\right].
\end{align*}
\end{lemma}
\begin{proof}
Let $\boldsymbol{v}_t^{s+1} := \frac{1}{b} \sum_{i_t\in {\cal I}_t} \left(\nabla f_{i_t}\left(\boldsymbol{x}_t^{s+1}\right) - \nabla f_{i_t}\left(\tilde{\boldsymbol{x}}^{s}\right)\right)$, then $\boldsymbol{u}_t^{s+1} = {\boldsymbol{v}}_t^{s+1} + \nabla f\left(\tilde{\boldsymbol{x}}^s\right)$.
Thus,
\begin{equation*}
    \begin{aligned}
    & \mathbb{E}\left[\left\|{\boldsymbol{u}}_t^{s+1}\right\|^2\right]&=&\ \ \ \ \mathbb{E}\left[\left\|{\boldsymbol{v}}_t^{s+1} - \left(\nabla f\left({\boldsymbol{x}}_t^{s+1}\right) - \nabla f\left(\tilde{\boldsymbol{x}}^s\right)\right) + \nabla f\left(\boldsymbol{x}_t^{s+1}\right)\right\|^2\right]\\
    & &\leq&\ \ \ \ 2 \mathbb{E}\left[\left\|\nabla f(\boldsymbol{x}_t^{s+1})\right\|^2\right] + 2\mathbb{E}\left[\left\|{\boldsymbol{v}}_t^{s+1} - \mathbb{E}\left[{\boldsymbol{v}}_t^{s+1}\right]\right\|^2\right]\\
    & &=&\ \ \ \ 2 \mathbb{E}\left[\left\|\nabla f\left({\boldsymbol{x}}_t^{s+1}\right)\right\|^2\right]+\frac{2}{b^2} \mathbb{E} \left[\left\|\sum_{i_t\in {\cal I}_t}\left(\nabla f_{i_t}\left({\boldsymbol{x}}_t^{s+1}\right)-\nabla f_{i_t}\left(\tilde{\boldsymbol{x}}^s\right) - \mathbb{E}\left({\boldsymbol{v}}_t^{s+1}\right) \right)\right\|^2 \right]\\
    & &\leq&\ \ \ \ 2 \mathbb{E}\left[\left\|\nabla f\left({\boldsymbol{x}}_t^{s+1}\right)\right\|^2\right] + \frac{2}{b}\mathbb{E}\left[\left\|\nabla f_{i_t}\left({\boldsymbol{x}}_t^{s+1}\right)-\nabla f_{i_t}\left(\tilde{\boldsymbol{x}}^s\right)\right\|^2\right]\\
    & &\leq&\ \ \ \ 2 \mathbb{E}\left[\left\|\nabla f\left({\boldsymbol{x}}_t^{s+1}\right)\right\|^2\right] + \frac{2L^2}{b} \mathbb{E}\left[\left\|{\boldsymbol{x}}_t^{s+1} - \tilde{\boldsymbol{x}}^s\right\|^2\right],
    \end{aligned}
\end{equation*}
where the first inequality holds for the basic inequality, i.e., $\|\boldsymbol{a}+\boldsymbol{b}\|^2 \leq 2(\|\boldsymbol{a}\|^2+\|\boldsymbol{b}\|^2)$ for any two vectors of the same sizes, the second inequality holds for the fact that $i_t$ are drawn uniformly randomly and independently from $\{1,2,\ldots,n\}$ and noting that for a random variable $\xi$, $\mathbb{E}[\|\xi - \mathbb{E}[\xi]\|^2] \leq \mathbb{E}[\|\xi\|^2]$, and the final inequality holds for the $L$-smoothness of $f_{i_t}$.
\end{proof}

\noindent Based on this lemma, we prove Lemma \ref{keylemma-minibatch} as follows.

\begin{proof}[\textbf{Proof of Lemma \ref{keylemma-minibatch}}]
By the $L$-smoothness of each $f_i$ (implying the $L$-smoothness of $f$), there holds
\begin{equation*}
    \begin{aligned}
        & \mathbb{E}\left[f\left({\boldsymbol{x}}_{t+1}^{s+1}\right)\right]&\leq&\ \ \ \mathbb{E}\left[f\left({\boldsymbol{x}}_t^{s+1}\right)\right]+\frac{L}{2}\mathbb{E}\left[\left\|{\boldsymbol{x}}_{t+1}^{s+1} - {\boldsymbol{x}}_t^{s+1}\right\|^2\right]+\mathbb{E}\left[\left\langle \nabla f\left({\boldsymbol{x}}_t^{s+1}\right),\ {\boldsymbol{x}}_{t+1}^{s+1} - {\boldsymbol{x}}_t^{s+1} \right\rangle\right]\\
        & &=&\ \ \ \mathbb{E}\left[f\left({\boldsymbol{x}}_t^{s+1}\right)\right] - b\eta_{\epsilon,s}\mathbb{E}\left[\left\|\nabla f\left({\boldsymbol{x}}_t^{s+1}\right)\right\|^2\right]+\frac{Lb^2 \eta_{\epsilon,s}^2}{2} \mathbb{E}\left[\left\|{\boldsymbol{x}}_t^{s+1}\right\|^2\right], \label{eq:keyineq1}
    \end{aligned}
\end{equation*}
where the equality holds for
\begin{equation*}
        \mathbb{E}\left[\langle \nabla f\left({\boldsymbol{x}}_t^{s+1}\right), {\boldsymbol{x}}_{t+1}^{s+1} - {\boldsymbol{x}}_t^{s+1}\rangle\right]=-b\eta_{\epsilon,s}\mathbb{E}\left[\left\langle \nabla f\left({\boldsymbol{x}}_t^{s+1}\right), \mathbb{E}\left[{\boldsymbol{x}}_t^{s+1}\right]\right\rangle\right]=-b\eta_{\epsilon,s} \mathbb{E}\left[\left\|\nabla f\left({\boldsymbol{x}}_t^{s+1}\right)\right\|^2\right].
\end{equation*}
Moreover, note that
\begin{equation*}
    \label{eq:keyineq2}
    \begin{aligned}
        & \mathbb{E}\left[\left\|{\boldsymbol{x}}_{t+1}^{s+1} - \tilde{\boldsymbol{x}}^s\right\|^2\right]&=&\ \ \ \mathbb{E}\left[\left\|{\boldsymbol{x}}_{t+1}^{s+1} - {\boldsymbol{x}}_{t}^{s+1} + {\boldsymbol{x}}_{t}^{s+1}-\tilde{\boldsymbol{x}}^s\right\|^2\right]\\
        & &=&\ \ \ \mathbb{E}\left[\left\|{\boldsymbol{x}}_{t+1}^{s+1} - {\boldsymbol{x}}_{t}^{s+1}\right\|^2\right] + \mathbb{E}\left[\left\|{\boldsymbol{x}}_{t}^{s+1}-\tilde{\boldsymbol{x}}^s\right\|^2\right]+2\mathbb{E}\left[\left\langle {\boldsymbol{x}}_{t+1}^{s+1} - {\boldsymbol{x}}_{t}^{s+1}, {\boldsymbol{x}}_{t}^{s+1}-\tilde{\boldsymbol{x}}^s\right\rangle\right]\\
        & &=&\ \ \ \mathbb{E}\left[\left\|{\boldsymbol{x}}_t^{s+1} - \tilde{\boldsymbol{x}}^s\right\|^2\right] + b^2 \eta^2_{\epsilon,s}\mathbb{E}\left[\left\|{\boldsymbol{u}}_t^{s+1}\right\|^2\right]-2b\eta_{\epsilon,s}\mathbb{E}\left[\left\langle \nabla f\left({\boldsymbol{x}}_t^{s+1}\right), {\boldsymbol{x}}_t^{s+1} - \tilde{\boldsymbol{x}}^s\right\rangle\right]\\
        & &\leq&\ \ \ \mathbb{E}\left[\left\|{\boldsymbol{x}}_t^{s+1} - \tilde{\boldsymbol{x}}^s\right\|^2\right] + b^2 \eta^2_{\epsilon,s}\mathbb{E}\left[\left\|{\boldsymbol{u}}_t^{s+1}\right\|^2\right]+2b\eta_{\epsilon,s} \left[\frac{1}{2\beta_s}\mathbb{E}\left[\left\|\nabla f({\boldsymbol{x}}_t^{s+1})\right\|^2\right] + \frac{\beta_s}{2} \left\|{\boldsymbol{x}}_t^{s+1} - \tilde{\boldsymbol{x}}^s\right\|^2\right]\\
        & &=&\ \ \ \left(1+b\eta_{\epsilon,s}\beta_s\right)\mathbb{E}\left[\left\|{\boldsymbol{x}}_t^{s+1} - \tilde{\boldsymbol{x}}^s\right\|^2\right]+b\eta_{\epsilon,s}\beta_s^{-1}\mathbb{E}\left[\left\|\nabla f({\boldsymbol{x}}_t^{s+1})\right\|^2\right] + b^2\eta^2_{\epsilon,s}\mathbb{E}\left[\left\|{\boldsymbol{u}}_t^{s+1}\right\|^2\right],
    \end{aligned}
\end{equation*}
where the third equality holds for the iterate of the proposed SVRG method, i.e., ${\boldsymbol{x}}_{t+1}^{s+1} = {\boldsymbol{x}}_t^{s+1} -b\eta_{\epsilon,s}{\boldsymbol{u}}_t^{s+1}$, and the inequality holds for the Young's inequality with some positive constant $\beta_s$.

Now consider the Lyapunov function
\begin{align*}
R_{t,s+1} := \mathbb{E}\left[f\left({\boldsymbol{x}}_{t}^{s+1}\right)\right] + c_{t,s} \mathbb{E}\left[\left\|{\boldsymbol{x}}_t^{s+1} - \tilde{\boldsymbol{x}}^s\right\|^2\right],
\end{align*}
where $c_{t,s}$ is specified in \eqref{eq:ct}.
By \eqref{eq:keyineq1} and \eqref{eq:keyineq2}, there holds
\begin{equation*}
    \begin{aligned}
        & R_{t+1,s+1}&\leq&\ \ \ \left(1+b\eta_{\epsilon,s}\beta_s\right)c_{t+1,s}\mathbb{E}\left[\left\|{\boldsymbol{x}}_t^{s+1}-\tilde{\boldsymbol{x}}_s\right\|^2\right]-b \eta_{\epsilon,s}\left(1-c_{t+1,s}\beta_s^{-1}\right) \mathbb{E}\left[\left\|\nabla f\left({\boldsymbol{x}}_t^{s+1}\right)\right\|^2\right]\\
        & & &\ \ \ +b^2\eta^2_{\epsilon,s}\left(c_{t+1,s} + \frac{L}{2}\right)\mathbb{E}\left[\left\|{\boldsymbol{u}}_t^{s+1}\right\|^2\right]+\mathbb{E}\left[f\left({\boldsymbol{x}}_t^{s+1}\right)\right].
    \end{aligned}
\end{equation*}
Plugging the bound of $\mathbb{E}\left[\left\|u_t^{s+1}\right\|^2\right]$ established in Lemma \ref{lemm:bound-ut} into the above inequality yields
\begin{align*}
R_{t+1,s+1}\leq \mathbb{E}\left[f\left({\boldsymbol{x}}_t^{s+1}\right)\right] + c_{t,s} \mathbb{E}\left[\left\|{\boldsymbol{x}}_t^{s+1} - \tilde{\boldsymbol{x}}^s\right\|^2\right]-b\eta_{\epsilon,s} \left[1-\frac{c_{t+1,s}}{\beta_s} - 2b\eta_{\epsilon,s}\left(c_{t+1,s}+\frac{L}{2}\right) \right] \mathbb{E}\left[\left\|\nabla f\left({\boldsymbol{x}}_t^{s+1}\right)\right\|^2\right],
\end{align*}
which concludes this lemma via the definition of $\Gamma_{t,s}$ \eqref{eq:Gammas}.
\end{proof}
\vspace{-0.75cm}
\section*{B. Proof of Theorem 1}
\label{appendix:Theorem 1}
\begin{proof}
By Lemma \ref{keylemma-minibatch}, for any $0\leq s\leq S-1$,
\begin{align*}
\sum_{t=0}^{m-1} \mathbb{E}\left[\left\|\nabla f\left({\boldsymbol{x}}_t^{s+1}\right)\right\|^2\right] \leq \frac{R_{0,s+1} - R_{m,s+1}}{\gamma_S}.
\end{align*}
Noting that ${\boldsymbol{x}}_0^{s+1} = \tilde{\boldsymbol{x}}^s$, ${\boldsymbol{x}}_m^{s+1} = \tilde{\boldsymbol{x}}^{s+1}$ and $c_{m,s}=0$, the above inequality implies
\begin{align*}
\sum_{t=0}^{m-1} \mathbb{E}\left[\left\|\nabla f\left({\boldsymbol{x}}_t^{s+1}\right)\right\|^2\right] \leq \frac{\mathbb{E}\left[f\left(\tilde{\boldsymbol{x}}^s\right)\right] - \mathbb{E}\left[f\left(\tilde{\boldsymbol{x}}^{s+1}\right)\right]}{\gamma_S}.
\end{align*}
Summing over $s$ from $0$ to $S-1$, the above inequality yields
\begin{align*}
\frac{1}{m S}\sum_{s=0}^{S-1}\sum_{t=0}^{m-1} \mathbb{E}\left[\left\|\nabla f\left({\boldsymbol{x}}_t^{s+1}\right)\right\|^2\right] \leq \frac{f\left(\tilde{\boldsymbol{x}}^0\right) - f\left({\boldsymbol{x}}^{*}\right)}{m  S\gamma_S}.
\end{align*}
Using the above inequality and the definition of ${\boldsymbol{x}}_{\mathrm{out}}$ in Algorithm \ref{alg:svrg-sbb}, we conclude this theorem.
\end{proof}
\vspace{-0.5cm}
\section*{C. Proof of Theorem \ref{svrg_sbb_minibatch}}
\label{appendix:Theorem 2}

To prove this theorem, we need the following lemma.
\begin{lemma}
\label{lemma:analytic}
Given some positive integer $l\geq 2$, then for any $0<x\leq\frac{1}{l}$, the following holds $(1+x)^l \leq e^{l x} \leq 1 + 2 l x$.
\end{lemma}
\begin{proof}
Note that $(1+x)^l = e^{l\cdot \ln(1+x)} \leq e^{lx}$,where the last inequality holds for $\ln(1+x)\leq x$ for any $x \in (0,1]$. Thus, we get the first inequality. Let $h(z) = 1+2z-e^z$ for any $z\in (0,1]$. It is easy to check that $h(z) \geq 0$ for any $z \in (0,1]$. Thus we get the second inequality.
\end{proof}

Based on this lemma, we show the proof of Theorem \ref{svrg_sbb_minibatch}.
\begin{proof}[\textbf{Proof of Theorem \ref{svrg_sbb_minibatch}}]
To prove this theorem, it only suffices to show that \eqref{Eq:cond-gamma} holds under the choice of $\beta_s$ \eqref{Eq:betas} and condition \eqref{Eq:cond-m-b} presented in this theorem.
To achieve this, we first provide two intermediate conditions implying \eqref{Eq:cond-gamma},
\begin{align}
& b\eta_{\epsilon,s}\beta_s + 2b\eta_{\epsilon,s}^sL^2 \leq \frac{1}{m}, \label{Eq:cond-inter1}\\
&b\eta_{\epsilon,s} \beta_s + b\eta_{\epsilon,s}L < \frac{1}{2}, \label{Eq:cond-inter2}
\end{align}
and then show that \eqref{Eq:cond-m-b} together with the choice of $\beta_s$ \eqref{Eq:betas} imply \eqref{Eq:cond-inter1}-\eqref{Eq:cond-inter2}.

\textbf{(a)} Here, we prove that \eqref{Eq:cond-inter1} and \eqref{Eq:cond-inter2} implies \eqref{Eq:cond-gamma}.
According to \eqref{eq:ct} and the initial condition ${c}_{m,s}=0$, we can easily check that for $t=m-1,\ldots,1,$
\begin{align*}
c_{t,s} = \frac{(\left(\rho_s\right)^{m-t}-1)\eta_{\epsilon,s}L^3}{\beta_s+2\eta_{\epsilon,s}L^2}.
\end{align*}
Noting that $\rho_s>1$ by its definition \eqref{eq:rhos}, then for any $t=1,\ldots, m-1$,
\begin{align}
\label{Eq:cts-bound}
c_{t,s} \leq c_{1,s} < \frac{(\left(\rho_s\right)^{m}-1)\eta_{\epsilon,s}L^3}{\beta_s+2\eta_{\epsilon,s}L^2}.
\end{align}
By \eqref{Eq:cond-inter1} and Lemma \ref{lemma:analytic},
\begin{align}
(\rho_s)^m
&= \left(1+b\eta_{\epsilon,s}\beta_s + 2b\eta_{\epsilon,s}^2L^2 \right)^m \nonumber\\
&\leq 1+2mb\eta_{\epsilon,s}(\beta_s + 2\eta_{\epsilon,s}L^2). \label{Eq:rhos-estimate}
\end{align}
Plugging \eqref{Eq:rhos-estimate} into \eqref{Eq:cts-bound}, and by the definition of \eqref{Eq:betas} imply
\begin{align}
\label{Eq:cond-ct-betas}
c_{t,s} < \frac{1}{2} \beta_s, \ t=1,\ldots,m, \ s=0,\ldots, S-1.
\end{align}
By \eqref{Eq:cond-ct-betas} and \eqref{Eq:cond-inter2}, there holds
\begin{align*}
b\eta_{\epsilon,s}(L+2c_{t+1,s})+c_{t+1,s}\beta_s^{-1}
<b\eta_{\epsilon,s}(L+\beta_s)+1/2<1.
\end{align*}
This yields \eqref{Eq:cond-gamma}.

\textbf{(b)}
In the next, we show that \eqref{Eq:cond-m-b} implies \eqref{Eq:cond-inter1} and \eqref{Eq:cond-inter2}.
By the definition of $\eta_{\epsilon,s}$ \eqref{Eq:eta-epsilon-s} and the $L$-smoothness of $f$, there holds
\begin{align}
\label{Eq:eta-epsilon-s-bound}
\frac{1}{mL}<\eta_{\epsilon,s}<\frac{1}{m\epsilon}.
\end{align}
Note that
\begin{align*}
&b\eta_{\epsilon,s}\beta_s + 2b\eta^2_{\epsilon,s}L^2
=4mb^2\eta_{\epsilon,s}^3L^3 + 2b\eta_{\epsilon,s}^2L^2
\leq\frac{1}{m}\left[\frac{4L^3}{\epsilon^3} \cdot \left(\frac{b}{m}\right)^2 + \frac{2L^2}{\epsilon^2} \cdot \frac{b}{m}\right]
< \frac{1}{m},
\end{align*}
where the first inequality holds for \eqref{Eq:eta-epsilon-s-bound}, and the final inequality follows from \eqref{Eq:cond-m-b} (i.e., $\frac{b}{m} < \frac{\epsilon^2}{L^2\left( 1+\sqrt{1+4\epsilon L^{-1}}\right)}$).
The above inequality yields \eqref{Eq:cond-inter1}.
Similarly, note that
\begin{equation}
    \label{Eq:cond2-verify}
    \begin{aligned}
    & & &\ \ b\eta_{\epsilon,s}\beta_s + b\eta_{\epsilon,s}L=4mb^2\eta_{\epsilon,s}^3L^3 + b\eta_{\epsilon,s}L\leq\frac{4L^3}{\epsilon^3}\cdot \left(\frac{b}{m}\right)^2 + \frac{L}{\epsilon} \cdot \left( \frac{b}{m} \right)<\frac{1}{4}<\frac{1}{2},
    \end{aligned}
\end{equation}
where the first inequality holds for \eqref{Eq:eta-epsilon-s-bound}, and the final inequality follows from \eqref{Eq:cond-m-b} (i.e., $b/m < \epsilon/2L(1+\sqrt{4L/\epsilon +1})$). The above inequality implies \eqref{Eq:cond-inter2}.

Furthermore, by \eqref{Eq:cond-ct-betas} and the definitions of $\Gamma_{t,s}$ \eqref{eq:Gammas} and $\beta_s$ \eqref{Eq:betas}, there holds
\begin{align*}
\Gamma_{t,s}
\geq b\eta_{\epsilon,s}\left[\frac{1}{2}-b\eta_{\epsilon,s}(L+\beta_s)\right]
\geq \frac{1}{4}b\eta_{\epsilon,s},
\end{align*}
where the second inequality follows from \eqref{Eq:cond2-verify}.
Thus,
\begin{align*}
\gamma_S
:= \min_{1\leq t\leq m,0\leq s\leq S-1} \Gamma_{t,s}
&\geq \frac{1}{4}b \min_{0\leq s \leq S-1}\eta_{\epsilon,s}.
\end{align*}
By Theorem 1, we conclude \eqref{eq:mini-rate}.
\end{proof}

\section*{D. Proof of Theorem \ref{svrg_sbb_pl}}
\begin{proof}
    According to Theorem \ref{svrg_sbb_minibatch}, the following holds
    begin
    \begin{equation}
        \label{Eq:gradient}
        \begin{aligned}
        \mathbb{E}[\|\nabla f(\tilde{{\boldsymbol{x}}}^{k})\|^2] \leq
         \frac{4\mathbb{E}[f(\tilde{{\boldsymbol{x}}}^{k-1})-f({\boldsymbol{x}}^*)]}{m S b \eta_{\epsilon,\min}}\leq \frac{2\mathbb{E}[\|\nabla f(\tilde{\boldsymbol{x}}^{k-1})\|^2]}{\lambda m S b \eta_{\epsilon,\min}}
        \leq  \rho \mathbb{E}[\|\nabla f(\tilde{\boldsymbol{x}}^{k-1})\|^2]
        \end{aligned}
    \end{equation}

   where the second inequality holds for the PL property. The above inequality implies \eqref{eq:pl-inner-rate}.
   By the PL property of $f$ again, we have
    \[
        \mathbb{E}[\|\nabla f(\tilde{{\boldsymbol{x}}}^{k})\|^2]\geq 2\lambda\mathbb{E}[f(\tilde{{\boldsymbol{x}}}^{k})-f({\boldsymbol{x}}^*)].
    \]
    This together with \eqref{Eq:gradient} yields \eqref{eq:pl-global}.
\end{proof}


Next, we provide some details including the convex and non-convex formulation of \textit{GOE}, some discussion of the \textit{GOE} framework and the proof for the lemma, propositions and theorems we proposed in the main paper. The numbering of equations and the reference follows that of the main paper.

\subsection*{E. Convex and Non-convex Formulation of GOE}
First, we revise some existing classification based ordinal embedding methods and verify that they are all the specific cases of generalized ordinal embedding \eqref{opt:nonconvex_goe}. As the dissimilarity functions in these method are all adopted as the squared Euclidean distance $\|\cdot\|^2_2$, we adopt the matrix $\boldsymbol{X}\in\mathbb{R}^{p\times n}$ to represent the embedding $\mathcal{X}$ for description clarity.

First of all, we introduce the Gram matrix $\boldsymbol{G}$ to linearize the distance calculation. We assume the matrix $\boldsymbol{X}=\{\boldsymbol{x}_1, \dots, \boldsymbol{x}_n\}\in\mathbb{R}^{p\times n}$ is centered at the origin as
\begin{equation}
  \label{eq:center_assum}
  \sum_{i=1}^n\ \boldsymbol{x}_i = \boldsymbol{0},
\end{equation}
and define the centering matrix $\boldsymbol{C}$
\begin{equation}
  \label{eq:center_matrix}
  \boldsymbol{C}= \boldsymbol{I}-\frac{1}{n}\boldsymbol{1}\cdot\boldsymbol{1}^\top,
\end{equation}
where $\boldsymbol{1}$ is a $n$-dimensional all-one column vector. With \eqref{eq:center_assum}, we immediately have $\boldsymbol{XC}=\boldsymbol{X}$. We call the Gram matrix $\boldsymbol{G}=\boldsymbol{X}^\top\boldsymbol{X}$ is also ``centered'' if $\boldsymbol{X}$ is a centered matrix which satisfies \eqref{eq:center_assum}, and it holds that
\begin{equation}
  \label{eq:center_gram}
  \boldsymbol{G}=\boldsymbol{X}^\top\boldsymbol{X}=\boldsymbol{C}^\top\boldsymbol{X}^\top\boldsymbol{XC}=\boldsymbol{C}^\top\boldsymbol{GC}.\\
\end{equation}
Given a centered matrix $\boldsymbol{X}$, we can establish a bijection between the Gram matrix $\boldsymbol{G}$ and the squared Euclidean distance matrix $\boldsymbol{D}$ as
\begin{subequations}
  \label{eq:bi_gram_distance}
  \begin{align}
    \boldsymbol{G}\ &=\ -\frac{1}{2}\boldsymbol{CDC},\\
    \boldsymbol{D}\ &=\ \textit{diag}(\boldsymbol{G})\cdot\boldsymbol{1}^\top-2\boldsymbol{G}+\boldsymbol{1}\cdot\textit{diag}(\boldsymbol{G})^\top,
  \end{align}
\end{subequations}
where the diagonal of $\boldsymbol{G}$ composes of the column vector $\textit{diag}(\boldsymbol{G})$. We refer \cite{dattorro2005convex} for the further properties of the squared Euclidean matrix $\boldsymbol{D}$. The bijection \eqref{eq:bi_gram_distance} indicates that
\begin{equation}
  \label{eq:distance_gram_elem}
  d^2(\boldsymbol{x}_i, \boldsymbol{x}_j) = g_{ii}-2g_{ij}+g_{jj},
\end{equation}
where $\boldsymbol{x}_i$ is the $i^{th}$ column of $\boldsymbol{X}$, $g_{ij}$ is the $(i,j)$ element of $\boldsymbol{G}$. Then, we can express the partial order $\{(\phi_{ij},\phi_{lk})\}$ or $\{(d^2_{ij}, d^2_{lk})\}$ as linear inequalities on the Gram matrix:
\begin{subequations}
  \begin{align}
    d^2_{ij} &< d^2_{lk}\ \ \\
    \Leftrightarrow g_{ii}-2g_{ij}+g_{jj} &< g_{ll}-2g_{lk}+g_{kk}.\label{eq:gram_constraint}
  \end{align}
\end{subequations}
Then we rewrite the \textit{GOE} problem \eqref{opt:nonconvex_goe} as
\begin{equation*}
  \underset{\boldsymbol{G}\in\mathbb{S}^{n}_+,\ \textit{rank}(\boldsymbol{G})\leq p}{\arg\min}\ \mathcal{L}_{\mathcal{Q},h}(\boldsymbol{G}, \mathcal{Y}_\mathcal{Q}), \tag{\ref{opt:convex_goe}}
\end{equation*}
where $\mathbb{S}^{n}_+$ is the $n$-dimensional positive semi-definite cone, the set of all symmetric positive semidefinite matrices in $\mathbb{R}^{n\times n}$; the rank constraint comes from the fact that $\textit{rank}(\boldsymbol{G})\leq \textit{rank}(\boldsymbol{X})\leq \min(n,p)=p$. With $\boldsymbol{G}=\boldsymbol{X}^\top\boldsymbol{X}$, $\boldsymbol{X}$ can be determined if we obtained $\boldsymbol{G}$ up to a unitary transformation.
Since the deviation between $d^2_{ij}$ and $d^2_{lk}$, $d^2_{ij}-d^2_{lk}$ which is the input of classifier $h$, can be calculated by the linear operations on $\boldsymbol{G}$, and the constraints like \eqref{eq:gram_constraint} are all linear, \eqref{opt:convex_goe} is always a convex optimization if we choose the convex classifier $h$ and the convex loss function $\ell$. This is the main advantage of optimizing $\boldsymbol{G}$ instead of $\boldsymbol{X}$.
We note
\begin{equation}
  \begin{aligned}
    & \Delta_q\boldsymbol{G} &=&\ \ \langle\boldsymbol{W}_q, \boldsymbol{G}\rangle = \textit{tr}(\boldsymbol{W}_q\boldsymbol{G})\\
    & &=&\ \ g_{ii}-2g_{ij}+g_{jj} - g_{ll}+2g_{lk}-g_{kk}
  \end{aligned}
\end{equation}
where $\boldsymbol{W}_q$ is a $n\times n$ matrix and has zero entry everywhere except on the entries corresponding to $q=(i,j,l,k)$ which has the form
\begin{equation}
  \boldsymbol{W}_q =
  \begin{blockarray}{crrrr}
   &i&j&l&k\\
    \begin{block}{c(rrrr)}
    i&1&-1&0&0\\
    j&-1&1&0&0\\
    l&0&0&-1&1\\
    k&0&0&1&-1\\
    \end{block}
  \end{blockarray}\ \ ,
\end{equation}
and $\langle \boldsymbol{W}_q, \boldsymbol{G}\rangle\triangleq\textit{tr}(\boldsymbol{W}_q\boldsymbol{G})=\textit{vec}(\boldsymbol{W}_q)^\top\textit{vec}(\boldsymbol{G})$ for any compatible matrices.

The Gram matrix $\boldsymbol{G}$ is introduced by the well-known Generalized Non-metric Multidimensional Scaling (\textit{GNMDS}) \cite{agarwal2007generalized}. \textit{GNMDS} obtains the embedding $\boldsymbol{X}$ by a ``SVM''-type algorithm. It adopts the soft-margin classifier and the hinge loss in \eqref{opt:convex_goe} as
\begin{equation}
  \label{opt:gnmds_convex}
  \begin{aligned}
    & &\underset{\{\xi_q\},\ \boldsymbol{G}}{\text{minimize}}&\ \ \underset{q\in\mathcal{Q}}{\sum}\ \xi_q,\\
    & &\text{subject to}&\ \ \Delta_q\boldsymbol{G}\leq 1 - \xi_q,\ \xi_q \geq 0\\
    & & &\ \ \boldsymbol{G}\succeq0,\ \textit{rank}(\boldsymbol{G})\leq p,\\
    & & &\ \sum_{i,j=1}^{n}\ g_{ij}=0,
  \end{aligned}
\end{equation}
where $\sum_{i,j=1}^{n}\ g_{ij}=0$ is the constraints for the `centered' $\boldsymbol{G}$ as $\boldsymbol{G}=\boldsymbol{X}^\top\boldsymbol{X}$ and
\begin{equation}
  0 = \left(\sum_{i=1}^n\boldsymbol{x}_i\right)^\top\left(\sum_{i=1}^n\boldsymbol{x}_i\right)=\sum_{i=1}^n\sum_{j=1}^n\boldsymbol{x}_i^\top\boldsymbol{x}_j=\sum_{i=1}^n\sum_{j=1}^ng_{ij}.
\end{equation}
The Crowd Kernel Learning (\textit{CKL}) \cite{tamuz2011adaptiive}, Stochastic Triplet Embedding (\textit{STE}) and \textit{t}-Distributed STE (\textit{TSTE}) \cite{vandermaaten2012stochastic} solve the \textit{GOE} problem by employing probabilistic models
\begin{equation}
  p^{\text{ckl}}_q = \frac{\exp(d^2_{lk})}{\exp(d^2_{ij})+\exp(d^2_{lk})},
\end{equation}
\begin{equation}
  p^{\text{ste}}_q = \frac{\exp(-d^2_{ij})}{\exp(-d^2_{ij})+\exp(-d^2_{lk})},
\end{equation}
and
\begin{equation}
  p^{\text{tste}}_q = \frac{\left(1+\frac{d^2_{ij}}{\alpha}\right)^{-\frac{\alpha+1}{2}}}{\left(1+\frac{d^2_{ij}}{\alpha}\right)^{-\frac{\alpha+1}{2}}+\left(1+\frac{d^2_{lk}}{\alpha}\right)^{-\frac{\alpha+1}{2}}}
\end{equation}
with threshold $t=0.5$ as the classifiers and the logistic loss like
\begin{equation}
  \label{opt:ckl_convex}
  \begin{aligned}
    & &\underset{\boldsymbol{G}}{\text{minimize}}&\ \ \underset{q\in\mathcal{Q}}{\sum}\ \log(1+\text{kernel}(\Delta_q\boldsymbol{G}))\\
    & &\text{subject to}&\ \ \boldsymbol{G}\succeq0,\ \textit{rank}(\boldsymbol{G})\leq p,\\
    & & &\ \sum_{i,j=1}^{n}\ g_{ij}=0.
  \end{aligned}
\end{equation}
where $\text{kernel}(\cdot)$ can be adopted as the Gaussian kernel and the Student-t kernel with $\alpha$ degrees of freedom.

Although the semi-definite positive programming \eqref{opt:convex_goe} is a convex optimization problem, there exist some disadvantages on obtaining the embedding $\boldsymbol{X}$ from the Gram matrix $\boldsymbol{G}$: (i) the positive semi-definite constraint on Gram matrix, $\boldsymbol{G}\succeq 0$, needs project $\boldsymbol{G}$ onto PSD cone $\mathbb{S}_+$, which is performed by the expensive singular value decomposition in each iteration due to the subspace spanned by the non-negative eigenvectors satisfies the constraint, is a computational bottleneck of optimization; (ii) the embedding dimension is $p\ll n$ and we hope that $\textit{rank}(\boldsymbol{G})\leq p$. If $\text{rank}(\boldsymbol{G})\gg p$, the freedom degree of $\boldsymbol{G}$ is much larger than $\boldsymbol{X}$ with over-fitting. Although $\boldsymbol{G}$ is a global optimal solution of \eqref{opt:convex_goe}, the subspace spanned by the largest $p$ eigenvectors of $\boldsymbol{G}$ also produce less accurate embedding. We can tune the regularization parameter $\lambda$ to force $\{\boldsymbol{G}_t\},\ t=1,\dots,T$ generated by the optimization algorithms to be low-rank and cross-validation is the most utilized technology. This also needs extra computational cost. In summary, projection and parameter tuning render gradient descent methods computationally prohibitive for learning the embedding $\boldsymbol{X}$ with ordinal information $\mathcal{Q}$. To overcome these challenges, we will exploit the non-convex and stochastic optimization techniques for the ordinal embedding problem. Here the non-convex formulation of ordinal embedding will replace the Gram matrix $\boldsymbol{G}$ with the distance matrix $\boldsymbol{D}$ which directly solving the embedded matrix $\boldsymbol{X}$
\begin{equation*}
  \underset{\boldsymbol{X}\in\mathbb{R}^{p\times n}}{\arg\min}\ \mathcal{L}_{\mathcal{Q},h}(\boldsymbol{X},\mathcal{Y}_{\mathcal{Q}}),\tag{\ref{opt:nonconvex_goe}}
\end{equation*}
where the loss and the classifier can be adopted as the same as the convex formulation. The instance for classier $h$ in \eqref{opt:nonconvex_goe} is $\Delta_q \boldsymbol{D}(\boldsymbol{X})$
\begin{equation}
  \begin{aligned}
    & \Delta_q \boldsymbol{D}(\boldsymbol{X}) &=&\ \ d^2_{ij}-d^2_{lk}\\
    & &=&\ \ \|\boldsymbol{x}_{i}-\boldsymbol{x}_{j}\|^2_2-\|\boldsymbol{x}_{l}-\boldsymbol{x}_{k}\|^2_2.
  \end{aligned}
\end{equation}

\section*{F. Discussion of GOE}
The existed literatures \cite{agarwal2007generalized,tamuz2011adaptiive,vandermaaten2012stochastic,53e99af7b7602d97023851bf,2015arXiv150102861A,Terada2014LocalOE,amid2015multiview} are all obtained the embedding $\mathcal{X}$ in the Euclidean space as it assume that the embedding $\mathcal{X}$ is lack of prior knowledge of $\mathcal{O}$ or the true dissimilarity function $\psi$ is unknown. For fair comparison, we focus on embedding $\mathcal{O}$ into Euclidean space $\mathbb{R}^p$ by classifying the instances indexed by $\mathcal{Q}$. That is to say, $\mathcal{X}\subset\mathbb{R}^p$ and the dissimilarity function $\phi$ of $\mathcal{X}$ is chosen as the squared Euclidean distance $d^2_{ij}=d^2(\boldsymbol{x}_i,\boldsymbol{x}_j)=\|\boldsymbol{x}_i-\boldsymbol{x}_j\|^2_2$. It is a very interesting direction to find helpful constraints in $\mathcal{O}$ or the prior knowledge of the true dissimilarity function $\psi$ in real applications. We leave this as one of our future works.

In practice, the class label set $\mathcal{Y}_{\mathcal{Q}}$ are typically obtained by the crowdsourcing system or questionnaire survey. The comparison results are inferred by combining answers from multiple human annotators. So these ordinal constraints could not be consistence with the true dissimilarity relationship of $\mathcal{O}$, which make $\mathcal{Y}_{\mathcal{Q}}$ contain noise. Here we focus on the problem which is provided single class label $y_q$ with a quadruplet $(i,j,l,k)$ to obtain the embedding. Using multiple inconsistence labels $\{y_q^1,\dots,y_q^s\}$ to estimate the embedding $\mathcal{X}$ is one of our future works.

The desired embedding dimension $p$ is a parameter of the ordinal embedding. It is well known that there exists a perfect embedding $\mathcal{X}$ estimated by any label set $\mathcal{Y}$ on the Euclidean distances in $\mathbb{R}^{n-2}$, even for the noisy constraints. Here we consider the \textit{low-dimensional} setting where $p\ll n$. The optimal or smallest $p$ for noisy ordinal constraints $\mathcal{Y}_{\mathcal{Q}}$ is another future work. The choices of $p$ in experiment section differ form the applications.

We notice that the label set $\mathcal{Y}_{\mathcal{Q}}$ carries the distance comparison information about $\boldsymbol{D}$, but $\boldsymbol{D}$ is invariant to the so-called \textit{similarity transformations, isotonic transformations or Procrustes transformations.} It means that the embedding $\boldsymbol{X}\in\mathbb{R}^{p\times n}$ obtained by $\mathcal{Y}_{\mathcal{Q}}$ is not unique as we rotate, reflect, translate, or scale $\boldsymbol{X}$ in Euclidean space and the new matrix ${\boldsymbol{X}}'$ also consist with the same constraints $\mathcal{Y}_{\mathcal{Q}}$. Without loss of generality, we assume the points $\boldsymbol{x}_1,\dots, \boldsymbol{x}_n$ are centered at the origin. Even disregarding similarity transformations, the ordinal embedding $\boldsymbol{X}$ is still not unique. Points of the embedding, $\{\boldsymbol{x}_1, \dots, \boldsymbol{x}_n\}$, can be perturbed slightly without changing their distance ordering, and so without violating any constraints. Kleindessner and von Luxburg \cite{53e99af7b7602d97023851bf} proved the long-standing conjecture that, under mind conditions, an embedding of a sufficiently large number of objects which preserves the total ordering of pairwise distances between all objects must place all objects to within $\varepsilon$ of their correct positions, where $\varepsilon\to0$ as $n\to\infty$. However, we focus on the dimension reduction setting as $p\ll n$ and $\mathcal{O}$ is always a finite set in real application. The embedding $\mathcal{X}$ or $\boldsymbol{X}$ could not be unique. Therefore, we adopt the classification metric to evaluate the quality of the estimated embedding.

\section*{F. Lipschitz Differentiability of GOE}
Note that the Lipschitz differentiability of the objective function is crucial for the establishment of the convergence rate of \textit{SVRG-SBB} in Theorem 1. In the following, we give a lemma to show that a part of aforementioned objective functions in the \textit{GOE} problem are Lipschitz differentiable.

\begin{lemma}
\label{lemma:lipschitz}
The ordinal embedding functions are Lipschitz differentiable for any bounded variable $\boldsymbol{X}$.
\end{lemma}

\begin{proof}
  Let
  \begin{equation}
  \boldsymbol{X}_q =
  \begin{pmatrix}
  \boldsymbol{X}_1 \\
  \boldsymbol{X}_2
  \end{pmatrix} =
  \begin{pmatrix}
  \boldsymbol{x}_i \\
  \boldsymbol{x}_j \\
  \boldsymbol{x}_l \\
  \boldsymbol{x}_k
  \end{pmatrix}
  \end{equation}
  where $\boldsymbol{X}_1=[\boldsymbol{x}_i^T,\boldsymbol{x}_j^T]^T$, $\boldsymbol{X}_2=[\boldsymbol{x}_l^T,\boldsymbol{x}_k^T]^T$ and
  \begin{equation}
  \boldsymbol{M} =
  \left(\begin{array}{cc}
  \boldsymbol{M}_1 & \\
  & \boldsymbol{M}_2 \\
  \end{array}\right) =
  \left(\begin{array}{rrrr}
  \boldsymbol{I}& -\boldsymbol{I}&  &   \\
  -\boldsymbol{I}& \boldsymbol{I}&  &   \\
  &  & -\boldsymbol{I}& \boldsymbol{I} \\
  &  &  \boldsymbol{I}& -\boldsymbol{I}
  \end{array}\right)
  \end{equation}
  where $\boldsymbol{I}_{d\times d}$ is the identity matrix.

  The first-order gradient of \textit{STE} is
  \begin{equation}
  \nabla_q f^{\text{ste}}(\boldsymbol{X})=
  2\frac{\exp(d^2_{ij}(\boldsymbol{X})-d^2_{lk}(\boldsymbol{X}))}{1+\exp(d^2_{ij}(\boldsymbol{X})-d^2_{lk}(\boldsymbol{X}))}\boldsymbol{MX}_q
  \end{equation}
  and the second-order Hessian matrix of \textit{STE} is
  \begin{equation}
  \begin{aligned}\nabla^2_q f^{\text{ste}}(\boldsymbol{X})=\frac{4\exp(d^2_{ij}(\boldsymbol{X})-d^2_{lk}(\boldsymbol{X}))}{[1+\exp(d^2_{ij}(\boldsymbol{X})-d^2_{lk}(\boldsymbol{X}))]^2}\boldsymbol{MX}_q\boldsymbol{X}^T_q\boldsymbol{M}+\frac{2\exp(d^2_{ij}(\boldsymbol{X})-d^2_{lk}(\boldsymbol{X}))}{1+\exp(d^2_{ij}(\boldsymbol{X})-d^2_{lk}(\boldsymbol{X}))}\boldsymbol{M}.
  \end{aligned}
  \end{equation}
  All elements of $\nabla^2_q f^{\text{ste}}(\boldsymbol{X})$ are bounded. So the eigenvalues of $\nabla^2_q f^{\text{ste}}(\boldsymbol{X})$ are bounded. By the definition of Lipschitz continuity, \textit{STE} has Lipschitz continuous gradient with bounded $\boldsymbol{X}$. As the loss function of \textit{CKL} is very similar to \textit{STE}, we omit the derivation of \textit{CKL}.

  The first-order gradient of \textit{GNMDS} is
  \begin{equation}
  \nabla_q f^{\text{gnmds}}(\boldsymbol{X})=
  \left\{\begin{array}{cl}
  \boldsymbol{0}, &\ \text{if}\ d^2_{ij}(\boldsymbol{X})+1-d^2_{lk}(\boldsymbol{X})<0,\\
  2\boldsymbol{MX}_q, &\ \text{otherwise},
  \end{array}\right.
  \end{equation}
  and the second-order Hessian matrix of GNMDS is
  \begin{equation}
  \nabla^2_q f^{\text{gnmds}}(\boldsymbol{X})=
  \left\{\begin{array}{cl}
  \boldsymbol{0}, & \text{if}\ d^2_{ij}(\boldsymbol{X})+1-d^2_{lk}(\boldsymbol{X})<0,\\
  2\boldsymbol{M}, &\ \text{otherwise}.
  \end{array}\right.
  \end{equation}
  If $d^2_{ij}(\boldsymbol{X})+1-d^2_{lk}(\boldsymbol{X})\neq 0$ for all $q\in\mathcal{Q}$, $\nabla_q f^{\text{gnmds}}(\boldsymbol{X})$ is continuous on $\{\boldsymbol{x}_i, \boldsymbol{x}_j, \boldsymbol{x}_l, \boldsymbol{x}_k\}$ and the Hessian matrix $\nabla^2_q f^{\text{gnmds}}(\boldsymbol{X})$ has bounded eigenvalues. So \textit{GNMDS} has Lipschitz continuous gradient in some quadruple set as $\{\boldsymbol{x}_i, \boldsymbol{x}_j, \boldsymbol{x}_l, \boldsymbol{x}_k\}\subset\mathbb{R}^{p \times 4}$. As the special case of $q=(i,j,l,k)$ which satisfied $d^2_{ij}(\boldsymbol{X})+1-d^2_{lk}(\boldsymbol{X})=0$ is exceedingly rare, we split the embedding $\boldsymbol{X}$ into pieces $\{\boldsymbol{x}_i, \boldsymbol{x}_j, \boldsymbol{x}_l, \boldsymbol{x}_k\}$ and employ \textit{SGD}, \textit{SVRG} and \textit{SVRG-SBB} to optimize the objective function of \textit{GNMDS}. The empirical results are showed in the experiment section.

  Note $s_{ij} = \alpha+d^2_{ij}(\boldsymbol{X})$, and the first-order gradient of \textit{TSTE} is
  \begin{equation}
    \label{eq:gradient:student}
    \begin{aligned}
    & \nabla_q f^{\text{tste}}(\boldsymbol{X})=\frac{\alpha+1}{s_{ij}}\left(\begin{array}{cc}\boldsymbol{M}_1\boldsymbol{X}_1 & \\ & \boldsymbol{O} \\ \end{array}\right)-\frac{\alpha^{-\alpha}(\alpha+1)}{s_{ij}^{-\frac{\alpha+1}{2}}+s_{lk}^{-\frac{\alpha+1}{2}}}\left(\begin{array}{cc}s_{ij}^{-\frac{\alpha+3}{2}}\boldsymbol{M}_1\boldsymbol{X}_1 & \\ & s_{lk}^{-\frac{\alpha+3}{2}}\boldsymbol{M}_2\boldsymbol{X}_2 \\ \end{array}\right)
    \end{aligned}
  \end{equation}
  The second-order Hessian matrix of TSTE is
    \begin{equation}
    \label{eq:hessian:student}
    \begin{aligned}
    & \nabla^2_q f^{\text{tste}}(\boldsymbol{X}) &=&\ \ \ \frac{\alpha+1}{s_{ij}^2}\left(\begin{array}{cc}s_{ij}\boldsymbol{M}_1-2\boldsymbol{M}_1\boldsymbol{X}_1\boldsymbol{X}^T_1\boldsymbol{M}_1 & \\ & \boldsymbol{O} \\ \end{array}\right)-\frac{\alpha^{-\alpha}(\alpha+1)}{s_{ij}^{-\frac{\alpha+1}{2}}+s_{lk}^{-\frac{\alpha+1}{2}}}\left(\begin{array}{cc}s_{ij}^{-\frac{\alpha+3}{2}}\boldsymbol{M}_1 & \\ & s_{lk}^{-\frac{\alpha+3}{2}}\boldsymbol{M}_2\end{array}\right)\\
    & &+&\ \ \ \frac{\alpha^{-\alpha}(\alpha+1)^2}{\left(s_{ij}^{-\frac{\alpha+1}{2}}+s_{lk}^{-\frac{\alpha+1}{2}}\right)^2}\left(\begin{array}{cc}s_{ij}^{-(\alpha+3)}\boldsymbol{M}_1\boldsymbol{X}_1\boldsymbol{X}^T_1\boldsymbol{M}_1 & \\ & s_{lk}^{-(\alpha+3)}\boldsymbol{M}_2\boldsymbol{X}_2\boldsymbol{X}^T_2\boldsymbol{M}_2\end{array}\right)\\
    & &+&\ \ \ \frac{\alpha^{-\alpha}(\alpha+1)(\alpha+3)}{s_{ij}^{-\frac{\alpha+1}{2}}+s_{lk}^{-\frac{\alpha+1}{2}}}\left(\begin{array}{cc}s_{ij}^{-\frac{\alpha+5}{2}}\boldsymbol{M}_1\boldsymbol{X}_1\boldsymbol{X}_1^T\boldsymbol{M}_1 & \\ & s_{lk}^{-\frac{\alpha+5}{2}}\boldsymbol{M}_2\boldsymbol{X}_2\boldsymbol{X}_2^T\boldsymbol{M}_2\end{array}\right)
    \end{aligned}
    \end{equation}
  The boundedness of eigenvalues of The loss function of $\nabla^2_p f^{\text{tste}}(\boldsymbol{X})$ can infer that the \textit{TSTE} loss function has Lipschitz continuous gradient with bounded $\boldsymbol{X}$.

  We focus on  a special case of quadruple comparisons as $i=l$ and $\{i,j,i,k\}\subset[n]^3$ in the Experiment section. To verify the Lipschitz continuous gradient of ordinal embedding objective functions with $c=(i,j,i,k)$ as $i=l$, we introduction the matrix $\boldsymbol{A}$ as
  \begin{equation}
  \boldsymbol{A} =
  \left(\begin{array}{rrrr}
  \boldsymbol{I} & \boldsymbol{0} & \boldsymbol{0} & \boldsymbol{0}\\
  \boldsymbol{0} & \boldsymbol{I} & \boldsymbol{0} & \boldsymbol{I}\\
  \boldsymbol{0} & \boldsymbol{0} & \boldsymbol{0} & \boldsymbol{I}
  \end{array}\right).
  \end{equation}
  By chain rule for computing the derivative, we have
  \begin{equation}
  \begin{aligned}
  & \nabla f_{ijk}(\boldsymbol{X}) &=&\ \ \ \boldsymbol{A}\nabla f_{ijlk}(\boldsymbol{X}),\\
  & \nabla^2 f_{ijk}(\boldsymbol{X}) &=&\ \ \boldsymbol{A}\nabla^2 f_{ijlk}(\boldsymbol{X})\boldsymbol{A}^T.
  \end{aligned}
  \end{equation}
  where $l = i$. As $\boldsymbol{A}$ is a constant matrix and $\nabla^2 f_{ijlk}(\boldsymbol{X})$ is bounded, all elements of the Hessian matrix $\nabla^2 f_{ijk}(\boldsymbol{X})$ are bounded. So the eigenvalues of $\nabla^2 f_{ijk}(\boldsymbol{X})$ is also bounded. The ordianl embedding functions of \textit{CKL}, \textit{STE} and \textit{TSTE} with triplewise compsrisons have Lipschitz continuous gradient with bounded $\boldsymbol{X}$.
\end{proof}

\end{document}